\documentclass{article}
\usepackage{iclr2027_conference,times}
\iclrfinalcopy
\usepackage{amsmath,amssymb,amsthm}
\usepackage{graphicx,booktabs,multirow,longtable,array,float}
\usepackage{microtype,xcolor,url,hyperref}
\newtheorem{theorem}{Theorem}
\newtheorem{proposition}{Proposition}
\newtheorem{lemma}{Lemma}
\newtheorem{corollary}{Corollary}

\DeclareMathOperator*{\softmax}{softmax}
\DeclareMathOperator{\sinc}{sinc}
\DeclareMathOperator{\Cov}{Cov}
\DeclareMathOperator{\Var}{Var}
\newcommand{\Om}{\Omega}
\newcommand{\mc}{\mathrm{MC}}
\newcommand{\std}{\mathrm{std}}
\newcommand{\E}{\mathbb{E}}
\newcommand{\R}{\mathbb{R}}
\graphicspath{{figures/}}
\hypersetup{pdftitle={Refinement Symmetry in Multimodal Transformers}}

\newcommand{\BudgetKL}{0.0168}

\title{Refinement Symmetry\\in Multimodal Transformers}
\author{Yuhao DU\textsuperscript{1,2}\quad Shunian CHEN\textsuperscript{1}\\
{\normalfont\textsuperscript{1}The Chinese University of Hong Kong, Shenzhen}\\
{\normalfont\textsuperscript{2}Shenzhen Loop Area Institute}\\
{\normalfont\texttt{yuhaodu1@link.cuhk.edu.cn}}}
\hypersetup{pdfauthor={Yuhao DU; Shunian CHEN}}

\begin{document}
\maketitle
\lhead{Preprint}
\begin{abstract}
Attention weights depend on token counts, which change with the representation of a signal.
We study refinement symmetry: splitting a representation while preserving content, position,
visible context, and total mass should preserve its contribution. Building on proportional and
quadrature attention, we show that split invariance forces the local mass factor to be linear
for any fixed positive attention kernel, provided that factor is nondecreasing. For changed representations, a physical coupling bounds attention
error by separating feature change from weight reallocation. In Qwen2.5-Omni-7B, duplicating
half the visual tokens threefold changes 255 of 3,586 MVBench answers under standard attention;
measure weighting preserves every answer under matched visibility. Under natural frame
resampling, it reduces distributional drift. At twofold merging
of a frozen video encoding, a five-seed evaluation shows an all-partition-correct accuracy gain
of 1.04 percentage points over global count weighting (average group mass) and 0.93 points over standard attention.
The advantage over global count also holds on WorldSense but depends on the compression budget.
The result is a representation principle with a measured benefit in robustness across partitions.
\end{abstract}
\section{Introduction}\label{sec:intro}
A token represents part of a signal, but standard attention assigns each key one unit of prior
weight. Sampling more densely or merging patches changes both representation and aggregation.
This is a general property of attention. In multimodal models, independent audio and video rates
set how many keys compete in the shared normalizer (Appendix~\ref{app:rate_panel}).

The dependence reaches predictions. Copying visual representations in
Qwen2.5-Omni-7B~\citep{xu2025qwen25omni} changes MVBench~\citep{li2024mvbench} answers while
keeping embeddings, positions, and visible context fixed; dividing each parent's mass among its
copies preserves every answer (Figure~\ref{fig:overview}).

\begin{figure}[t]
\centering\includegraphics[width=\textwidth]{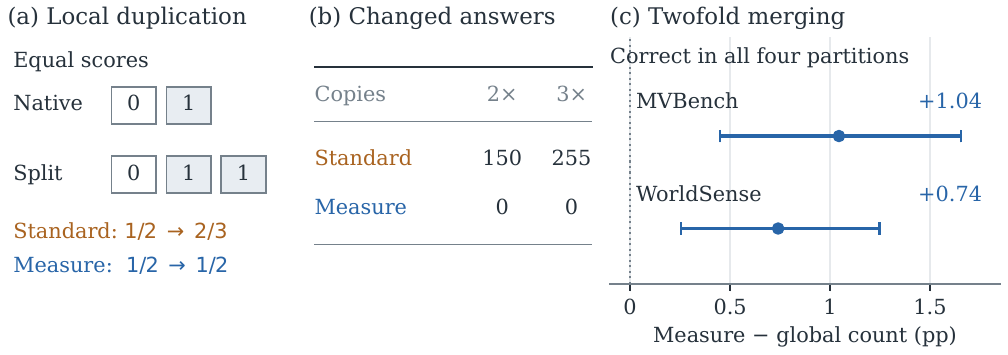}
\caption{\textbf{From token multiplicity to correctness across partitions.}
(a) With equal scores, duplicating only the value-one token shifts standard attention
from $1/2$ to $2/3$; dividing its mass between copies preserves $1/2$. (b) In Qwen2.5-Omni-7B, duplicating half the visual representations changes 150 and 255 of 3,586 MVBench answers
under standard attention, and none under measure weighting with matched positions and visibility.
(c) At twofold merging, measure weighting improves correctness under all four tested partitions
over global count weighting on two benchmarks. Points show mean paired effects across five training seeds; bars are
descriptive, unadjusted 95\% crossed source-video/seed bootstrap intervals.}
\label{fig:overview}
\end{figure}

We call this requirement \emph{refinement symmetry}: splitting represented mass while preserving
content, position, and visibility preserves aggregation. Token Merging already reproduces the contribution of copied keys through
proportional attention~\citep{bolya2023tome}; continuum and geometric attention use quadrature
weights~\citep{calvello2025continuum,berner2026operators,bonev2025sphere}.
We characterize this rule by a finite split requirement, then use a coupling of represented
regions to bound the effect of changing their features and scores. Ordinary text retokenization need not preserve these conditions.

We establish three results:
\begin{itemize}
\item \textbf{Finite characterization.} Invariance to every equal split uniquely selects
mass-proportional weights among monotone local mass rules with a fixed positive kernel.
\item \textbf{Coupled stability.} A physical mass coupling separates value perturbation from
reallocation, giving a sharp bound without an explicit token-count factor.
\item \textbf{Partition robustness.} At twofold merging, inherited group mass improves
all-partition-correct accuracy over global count on MVBench and WorldSense.
\end{itemize}

\section{Refinement symmetry and stability}\label{sec:theory}
\subsection{Represented mass and exact refinement}\label{sec:covariance}
Consider a video interval divided into cells $C_i$. Token $i$ carries content $h_i$ and represents
mass $m_i=\mu(C_i)>0$, such as duration under the measure $\mu$. Splitting a cell into children
whose masses sum to $m_i$, while copying its content and positional features, changes the
representation without adding information. If each copy also has the parent's visible context,
we require the aggregate to remain unchanged. We call this an exact duplicate split.
For a layer that outputs one representation per token, the corresponding requirement is that
children receive copies of their parent's output.

\paragraph{A minimal counterexample.}
For a fixed query, two equally scored tokens with values zero and one yield standard attention
output $1/2$. Splitting the second into two identical copies changes the output to $2/3$.
The represented content is unchanged, but its relative contribution grows. Assigning each copy
half the parent's mass keeps the output at $1/2$. In contrast, duplicating \emph{every} visible
key/value by the same factor leaves standard attention unchanged: the common factor cancels.
The issue is local multiplicity relative to other content, not token count alone.

We fix a reference measure $\mu$, so a cell's mass records how much of the signal it represents.
Different reference measures define different aggregation priors; the split requirement determines
the dependence on mass once this choice is fixed (Appendix~\ref{app:criterion}).

An additive mass corrects aggregation directly. For a positive kernel $K$, scores $s_i$, and values $v_i$, define
\begin{equation}\label{eq:ours}
 A_m^K(s,v)=\frac{\sum_i m_iK(s_i)v_i}{\sum_i m_iK(s_i)}.
 \qquad K(s)=e^s\ \Longrightarrow\ A_m^K=\sum_i\softmax(s+\log m)_i v_i.
\end{equation}
The positive-kernel form makes the source of invariance explicit: mass multiplies the kernel.
Any further score transform, such as temperature, is applied to $s$ before adding $\log m$.

\subsection{Characterization of the local weighting rule}\label{sec:split}
\begin{samepage}
\begin{proposition}[Split invariance and local uniqueness]\label{prop:split}
Equation~\eqref{eq:ours} preserves every duplicate split whose child masses sum to the parent's
mass. Among normalized rules with weights $w(m_i)K(s_i)$, where $K>0$ is fixed and
$w:(0,\infty)\to(0,\infty)$ is nondecreasing and depends only on local mass, invariance for every
integer equal split and every score/value configuration holds if and only if $w(m)=cm$, $c>0$.
\end{proposition}
\end{samepage}

The sufficiency identity is the proportional-attention correction~\citep{bolya2023tome}.
Necessity follows from $k w(m/k)=w(m)$, rational homogeneity, and monotonicity; sufficiency is
additivity. Thus $w(m)=m^\alpha$ requires $\alpha=1$. This characterizes the weighting rule
within its stated local class (Appendix~\ref{app:splitting}). Task training need not identify
this exponent: row-constant biases cancel, and shared type biases can confound it
(Appendix~\ref{app:identification}).

\subsection{Natural retokenization as a coupled perturbation}\label{sec:stability}
To compare two partitions, we match the portions of the signal that their tokens represent.
The error separates changes in token values from changes in how attention allocates weight
across those portions. Let $m_i$ and $n_j$ have the same
total mass $M$. A \emph{mass coupling} $\gamma_{ij}\ge0$ has row sums $m_i/M$ and column sums
$n_j/M$. For two partitions of one physical domain, a canonical choice is
$\gamma_{ij}=\mu(C_i\cap D_j)/M$: it compares representations of the same physical points.

\begin{theorem}[Approximate refinement stability]\label{thm:stability}
Compare $A=A_m^K(s,v)$ and $B=A_n^K(t,w)$ with finite positive kernel values. On pairs with $\gamma_{ij}>0$, define
\[
 \epsilon_v=\max\|w_j-v_i\|,\qquad
 \Delta=\max[\log K(t_j)-\log K(s_i)]-\min[\log K(t_j)-\log K(s_i)],
\]
and let $D_v=\max_{i,i'}\|v_i-v_{i'}\|$. Then, for any norm,
\begin{equation}\label{eq:stability}
 \boxed{\ \|B-A\|\le\epsilon_v+D_v\tanh(\Delta/4).\ }
\end{equation}
The coefficient of $D_v$ is sharp. If paired log-kernels differ by at most $\epsilon_s$,
the second term is at most $D_v\tanh(\epsilon_s/2)$.
\end{theorem}

$\Delta$ measures the spread of paired log-kernel changes; a common shift contributes zero.
The proof lifts both attention distributions to the same coupling. The coupled likelihood ratio
has logarithmic range $\Delta$, and a sharp Hilbert-distance inequality then bounds total
variation by $\tanh(\Delta/4)$; we give a proof in Appendix~\ref{app:tilt}, and the
inequality was previously obtained by \citet{cohen2023hilbert}. The two terms
account for changed representations and changed allocation, with no explicit token-count factor. The coupling specifies which features
represent the same physical region even when neither partition refines the other.

The bound covers changed token content through paired feature and score errors. Exact copies give $\epsilon_v=\Delta=0$. For point evaluations of Lipschitz value and
log-kernel fields, paired points within distance $r$ give
$\epsilon_v\le L_vr$ and $\Delta\le2L_sr$. Cell diameters and encoder approximation errors therefore
control distinct parts of the discrepancy.

\paragraph{Propagation through a network.}
Layerwise discrepancies propagate through the Lipschitz constants of subsequent blocks
(Appendix~\ref{app:network_bound}, Corollary~\ref{cor:network}). The network bound uses constants for complete blocks, including position, masks,
normalization, and residual branches; its magnitude depends on those constants. A readout logit error in $\ell_\infty$ below half the reference top-two margin preserves argmax.

\subsection{Sampling density and operator consistency}\label{sec:attn}
Mass quadrature also converges to the normalized continuum integral $\mathcal A_\mu^K$ when
scores and values are continuous, $K(s)>0$ is continuous, and cell diameters vanish on a compact
domain~\citep{calvello2025continuum,berner2026operators}.
If representative points instead converge to a nonuniform sampling density $\rho$ relative to
$\mu$, standard attention converges to $\mathcal A_{\rho\mu}^K$. Its bias is
$\Cov_\pi(v,\rho)/\E_\pi[\rho]$, where $d\pi\propto K(s)\,d\mu$.
Thus finer sampling need not remove an allocation error: the limit depends on whether density
correlates with values under the intended attention measure
(Appendix~\ref{app:quadrature}, Figure~\ref{fig:convergence}).

\section{Measure weighting in multimodal transformers}\label{sec:method}
\subsection{Reference calibration across modalities}\label{sec:gauge}
Refinement fixes how mass enters aggregation; modality priors and feature fidelity remain
separate choices. Audio and video measures have different units. Choose a measure $\mu_\tau$
and a positive scale $\kappa_\tau$ for each modality, giving token $i$ the effective mass
$m_i=\kappa_{\tau(i)}\mu_{\tau(i)}(C_i)$. Split invariance fixes linear dependence on cell
measure; the constants $\kappa_\tau$ set the relative modality prior. A common rescaling of all
co-visible types cancels, whereas changing one type's scale changes its odds against the others
(Appendix~\ref{app:calibration}).

For a checkpoint with uniform native cell measures $m_\tau^0$, choosing $\kappa_\tau=1/m_\tau^0$ gives
$\widetilde m_i=\mu_{\tau(i)}(C_i)/m_{\tau(i)}^0$. Fix these reference values before intervention.
Every native bias is then zero in exact arithmetic. Text and auxiliary keys outside the intervention
retain unit reference mass. This preserves their relative weight at calibration; recalibrating after
each tokenization would generally change it. Appendix~\ref{app:reference_impl} specifies the reference conventions.
Mass normalization of token losses is distinct from learning task weights (Appendix~\ref{app:loss}).

\subsection{Position and support geometry}\label{sec:rope}
We evaluate rotary position embeddings (RoPE; \citealp{su2021rope}) at physical coordinates.
The primary deployment comparison holds point positions fixed across weighting rules.
Support averaging is evaluated separately in the controlled temporal and pretrained-support studies.
For a support distribution $p$, let $U(p)=\mathbb E_{X\sim p}R(X)$ be the averaged rotary matrix.
Independent query and key supports give the positional logit $q^\top U(p_q)^\top U(p_k)k$.
This averages the logit, not its nonlinear softmax. Uniform intervals give a sinc attenuation;
variance-matched Gaussians agree to second order (Appendix~\ref{app:position},
Figure~\ref{fig:position}). Temporal quantization is tested in Appendix~\ref{app:lattice}.

For general supports, the change in the averaged rotary matrix is at most
$\min\{2,\omega_{\max}W_1(p,p')\}$ in Euclidean operator norm. Here $\omega_{\max}$ is the
largest angular frequency and $W_1$ the minimum expected support displacement over couplings.
This connects geometry to the score term of Theorem~\ref{thm:stability}; equal mass alone does
not control it (Appendix~\ref{app:support_bound}).

\subsection{Visibility and overlapping coverage}\label{sec:coverage}
For a fixed physical query $q$ with positive-mass visible region $V(q)$, use the visible mass
$m_i(q)=\mu(C_i\cap V(q))$. With copied scores and values, additivity preserves splits even when a causal boundary cuts a
cell. Copies with different token-index prefixes do not represent this same query.
For overlapping coverage, nonnegative functions $\psi_i$ with $\sum_i\psi_i=1$ assign mass
$\int_{V(q)}\psi_i\,d\mu$ to token $i$; zero visible masses are omitted. This integrates a
mixture of token kernels, retaining disagreements between overlapping features
(Appendix~\ref{app:coverage}; multiscale probe in Appendix~\ref{app:multiscale}).

\section{Experiments}\label{sec:experiments}
\subsection{Evaluation design}\label{sec:setup}
We distinguish attention allocation, answer identity, and task accuracy. Released-model probes use
Qwen2.5-Omni-7B on MVBench, with Audio-Flamingo-3 for controlled representation probes and
Qwen2.5-Omni-3B for a scale comparison~\citep{xu2025qwen25omni,goel2025af3,li2024mvbench}.
Trained deployment evaluations use MVBench and WorldSense~\citep{hong2026worldsense}.
Known-operator and temporal studies isolate approximation error and position.

Training seeds are the independent units for trained comparisons; inference comparisons pair
questions and cluster by source video where available. Repeated rates remain within these units.
We report 95\% intervals and control multiplicity within the 54-contrast family of controlled
studies and four prespecified deployment contrasts per benchmark; the controlled-study family is retrospective. Cellwise and
additional-budget comparisons are descriptive unless stated otherwise
(Appendices~\ref{app:estimands}, \ref{app:multiplicity}, and~\ref{app:exclusions}). Exact identities use
primitive and floating-point controls (Appendices~\ref{app:primitives} and~\ref{app:repro}).

\subsection{Content-preserving duplication}\label{sec:released}
We isolate multiplicity by duplicating a contiguous half of the visual representations while
preserving embeddings, positions, and incoming visibility. Standard attention changes 150 answers at split factor two and 255 at factor three;
measure weighting preserves every answer (Table~\ref{tab:duplication}). The answer changes include both gains and losses in
correctness.
When copies instead have different causal prefixes, measure weighting changes four of 720 answers
(0.56\%); those prefixes violate the exact symmetry's visibility condition.

For example, a recorded action-prediction question asks, ``What will the person do next?''
The native answer is the correct ``Put down the bag.'' Duplicating 512 of its 1,024 visual tokens
changes standard attention's answer to ``Take the food'' at both factors; measure weighting
keeps the correct answer. This instance illustrates a loss, while Table~\ref{tab:duplication}
reports both directions across all questions.

\begin{table}[!htbp]
\centering\small
\caption{\textbf{Same content, changed predictions.} Qwen2.5-Omni-7B on 3,586 usable MVBench questions (14 of 3,600 excluded). Half the visual tokens are replaced by the stated number of copies, with matched positions and visibility. C/W denote correct/incorrect; transitions describe standard attention. Measure weighting changes no answers.}
\label{tab:duplication}
\begin{tabular}{@{}rrrrrr@{}}
\toprule
Split & Standard changes & Measure changes & C$\to$W & W$\to$C & W$\to$W \\
\midrule
2 & 150 & 0 & 63 & 55 & 32 \\
3 & 255 & 0 & 111 & 90 & 54 \\
\bottomrule
\end{tabular}
\end{table}

Full-modality duplication by $k=2,3,4$ multiplies standard cross-modal attention odds by $k$ in
the float32 probe. Measure-weighted logit deviations remain at the measured arithmetic floor.
Audio-Flamingo-3 reproduces the odds identity on 40 independent representation draws;
Qwen2.5-Omni-3B shows the same answer-preservation pattern on the 3,586-question panel.
Appendices~\ref{app:duplication} and~\ref{app:qwen3b} specify these interventions.

\paragraph{Local allocation at fixed budget.}
A separate control allocates the same 16 frames unequally to two equal-duration halves of each
clip. For dense/sparse allocations $10/6$, $12/4$, and $14/2$, measure weighting scales the
first-layer dense-to-sparse attention odds by $6/10$, $4/12$, and $2/14$, respectively;
a single per-modality count correction leaves those odds unchanged. Of 60,480
clip--head--allocation observations (720 clips, 28 heads), 60,479 agree within 1\%; the exception
has extreme baseline odds. This identifies the allocation effect in the shared-score first
layer (Appendix~\ref{app:unequal}).

\subsection{Natural frame resampling}\label{sec:natural}
Natural frame resampling changes both token count and content. We test its effect on prediction
stability using a shared reference distribution. On 540 MVBench questions from 534 source videos, frame budgets 4, 8, 24, 32, and 48
use the shared native 16-frame option distribution as reference. The endpoint is
$\operatorname{KL}(p_{16}\|p_b)$, averaged over the five nonreference budgets within question. Measure weighting reduces mean reference-answer
Kullback--Leibler (KL) divergence by \BudgetKL\ nats (95\% video-clustered interval
$[0.0077,0.0259]$; Bonferroni-adjusted $p=0.016$; Figure~\ref{fig:budgets}).
The paired accuracy difference is $-0.30$\,pp with interval $[-1.16,+0.56]$; answer-flip
differences are also unresolved. Calibrated count and measure weighting coincide under these
uniform budget changes (Appendix~\ref{app:frame_budgets}). Qwen2.5-Omni-3B has a descriptive KL reduction of 0.0154 nats
(Appendix~\ref{app:qwen3b}).

\begin{figure}[!htbp]
\centering\includegraphics[width=\textwidth]{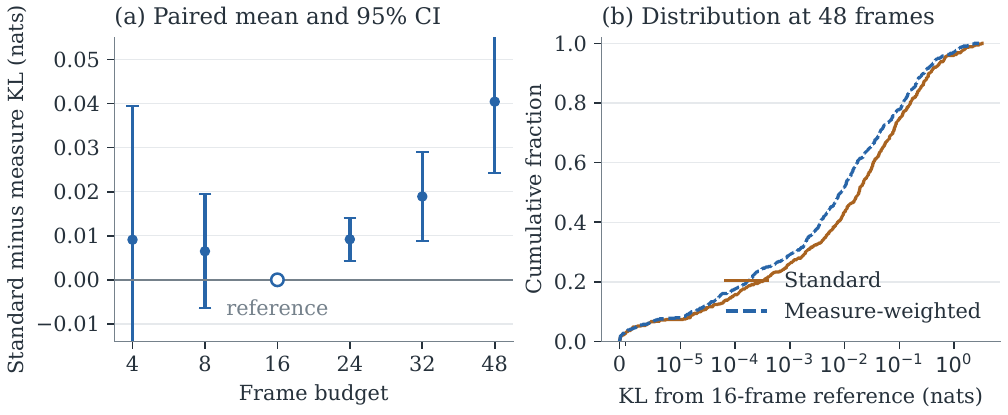}
\caption{\textbf{Distributional stability under natural resampling.} Qwen2.5-Omni-7B on MVBench.
(a) Standard minus measure-weighted $\operatorname{KL}(p_{16}^{\rm std}\|p_b^a)$ from the
native 16-frame option distribution, with paired 95\% intervals over 540 questions; the open
marker is the reference budget, zero by construction. (b) Empirical cumulative distributions
at 48 frames; the horizontal axis is linear near zero and logarithmic beyond $10^{-5}$ nats.
Per-budget intervals describe question variability; the primary repeated-budget test clusters by source video.}
\label{fig:budgets}
\end{figure}

\subsection{Adaptive merging at deployment}\label{sec:deployment}
The deployment test evaluates whether a single trained model remains correct across four
partitions of a frozen video encoding. Each clip has 16 native temporal patches, grouped into
$K\in\{8,4,2\}$ contiguous groups by uniform, motion, entropy, or keyframe boundaries. All groups
cover the same native patches and use mass-weighted feature averages. Three rules evaluate the
same grouped features and physical positions: standard attention adds no bias; global count
weighting adds $\log(16/K)$ to every visual key; measure weighting adds $\log|G_i|$ for a
group containing $|G_i|$ native patches. Thus measure weighting here is precisely groupwise
proportional attention~\citep{bolya2023tome}; the count comparator uses the average group size.
The 18-task MVBench panel evaluates five checkpoints, each trained at the native partition with
rank-8 low-rank adaptation (LoRA; 400 steps, 696 training questions), on 2,489 questions from 2,290 source videos.
Training and test sources are disjoint (Appendix~\ref{app:deployment_setup}); an untrained
support comparison appears in Appendix~\ref{app:pretrained_support}.

Robust accuracy is the fraction of questions answered correctly under all four evaluated
partition rules; it differs from the minimum of the four partition-averaged accuracies.

\begin{table}[!htbp]
\centering\small
\caption{\textbf{Robustness under post-encoder merging.} Qwen2.5-Omni-7B on 2,489 MVBench and 2,188 WorldSense questions; means $\pm$ sample SD over five training seeds. Mean accuracy averages uniform, motion, entropy, and keyframe partitions; robust accuracy requires correctness under all four (both in \%). JSD is mean pairwise Jensen--Shannon divergence in nats. Bold marks the best mean within each benchmark and budget; paired inference is in Table~\ref{tab:deployment-inference}.}
\label{tab:prospective}
\setlength{\tabcolsep}{4pt}
\begin{tabular}{@{}llrrr@{}}
\toprule
Reduction & Rule & Mean accuracy $\uparrow$ & Robust accuracy $\uparrow$ & JSD $\downarrow$ \\
\midrule
\multicolumn{5}{@{}l}{\emph{MVBench}} \\
2$\times$ & Standard & $64.13\pm0.65$ & $56.50\pm0.64$ & $0.0360\pm0.0046$ \\
 & Global count & $64.64\pm0.54$ & $56.39\pm0.48$ & $0.0413\pm0.0051$ \\
 & Measure & $\mathbf{64.82}\pm0.58$ & $\mathbf{57.44}\pm0.60$ & $\mathbf{0.0350}\pm0.0044$ \\
\midrule
4$\times$ & Standard & $61.26\pm0.59$ & $53.08\pm0.61$ & $\mathbf{0.0409}\pm0.0051$ \\
 & Global count & $\mathbf{62.41}\pm0.56$ & $52.53\pm0.82$ & $0.0512\pm0.0052$ \\
 & Measure & $62.32\pm0.42$ & $\mathbf{53.35}\pm0.44$ & $0.0456\pm0.0047$ \\
\midrule
8$\times$ & Standard & $57.49\pm0.34$ & $\mathbf{51.60}\pm0.48$ & $\mathbf{0.0272}\pm0.0033$ \\
 & Global count & $\mathbf{58.27}\pm0.73$ & $50.49\pm0.65$ & $0.0416\pm0.0050$ \\
 & Measure & $58.23\pm0.70$ & $50.54\pm0.84$ & $0.0387\pm0.0052$ \\
\midrule
\multicolumn{5}{@{}l}{\emph{WorldSense: primary budget}} \\
2$\times$ & Standard & $38.45\pm0.53$ & $31.52\pm0.93$ & $0.0333\pm0.0040$ \\
 & Global count & $38.55\pm0.67$ & $30.99\pm1.11$ & $0.0380\pm0.0039$ \\
 & Measure & $\mathbf{38.56}\pm0.59$ & $\mathbf{31.73}\pm0.89$ & $\mathbf{0.0331}\pm0.0040$ \\
\bottomrule
\end{tabular}
\end{table}

On MVBench, all rules coincide at the native partition, with 66.83\% accuracy. At the primary budget $K=8$,
measure weighting has the highest mean and all-partition-correct accuracy and the lowest
pairwise Jensen--Shannon divergence (JSD; Table~\ref{tab:prospective}).
Robust accuracy rises by 1.04\,pp over global count weighting and
0.93\,pp over standard attention. Four prespecified contrasts test robust accuracy and JSD
against the two comparators within each benchmark. Each uses the larger of a source-video sign-flip $p$-value and a
five-seed paired $t$-test $p$-value, with four-test Bonferroni correction. On MVBench, three resolve;
the JSD contrast against standard attention does not (Table~\ref{tab:deployment-inference}).
Appendix~\ref{app:deployment_inference} gives the procedure. The 0.69\,pp mean-accuracy
increase over standard attention is descriptive.

\begin{table}[!htbp]
\centering\small
\caption{\textbf{Primary paired effects at twofold reduction.} Qwen2.5-Omni-7B, five trained seeds per benchmark. Robust accuracy requires correctness under all four partitions. Effects pair the same checkpoints and test questions. Differences are measure minus comparator, computed before rounding: robust accuracy uses pp; Jensen--Shannon divergence (JSD) uses nats. Intervals are descriptive, unadjusted 95\% crossed source-video/seed bootstrap intervals; adjusted $p=\min(1,4\max(p_{\rm source},p_{\rm seed}))$ within each benchmark.}
\label{tab:deployment-inference}
\begin{tabular}{llrr}
\toprule
Comparator & Endpoint & Difference [95\% interval] & Adjusted $p$ \\
\midrule
\multicolumn{4}{l}{\emph{MVBench: 2,489 questions, 2,290 source videos}} \\
Global count & Robust acc. & +1.04 [+0.45, +1.65] & 0.0082 \\
Global count & JSD & -0.0063 [-0.0080, -0.0047] & $7.7\times10^{-4}$ \\
Standard & Robust acc. & +0.93 [+0.16, +1.69] & 0.024 \\
Standard & JSD & -0.0010 [-0.0029, +0.0008] & 0.90 \\
\midrule
\multicolumn{4}{l}{\emph{WorldSense: 2,188 questions, 1,158 source videos}} \\
Global count & Robust acc. & +0.74 [+0.26, +1.25] & 0.016 \\
Global count & JSD & -0.0049 [-0.0059, -0.0040] & $2\times10^{-4}$ \\
Standard & Robust acc. & +0.21 [-0.40, +0.82] & 1.0 \\
Standard & JSD & -0.0002 [-0.0014, +0.0010] & 1.0 \\
\bottomrule
\end{tabular}
\end{table}

\paragraph{Where the benefit occurs.}
Uniform groups give measure and global count weighting identical biases. Their difference is
therefore carried by adaptive partitions. At $K=8$, the keyframe partition is the hardest for
all rules; measure weighting raises its accuracy from 63.03\% (standard) and 63.78\% (global
count) to 64.07\% (Appendix~\ref{app:perpartition}). At fourfold reduction, the robust-accuracy
lead persists descriptively. At eightfold reduction, standard attention has higher robust
accuracy and lower JSD, while mean accuracies remain within 0.78\,pp. At deeper compression, the shared merged features depart further from the native representation;
Appendix~\ref{app:mechanism} reports the accompanying embedding and attention-allocation diagnostics.

The same protocol on WorldSense uses five separately trained checkpoints and 2,188 test
questions from 1,158 videos. Against global count weighting, robust accuracy rises by
0.74\,pp (adjusted $p=0.016$) and JSD falls by 0.0049 nats ($p=2\times10^{-4}$).
Neither contrast against standard attention resolves
(Appendix~\ref{app:worldsense}). The cross-benchmark evidence is therefore strongest for using
local inherited mass over one average mass for all groups. Exploratory calibration,
frame selection, fine-tuning, and cross-rate results appear in
Appendices~\ref{app:ws_calibration}, \ref{app:adaptive}, \ref{app:trained}, and~\ref{app:deployment}.

\paragraph{Computational cost.}
On one NVIDIA H20, eightfold merging reduces mean decoder tokens from 1,102 to 206 and decoder-forward
latency from 0.29--0.30\,s to 0.073\,s, a 3.9--4.2-fold acceleration across rules.
Summed pipeline-stage time falls from 4.69--4.71\,s to 4.44\,s (5.2--5.6\%); peak allocated
memory remains 19.34\,GiB because the frozen dense encoder and acquisition costs remain.
These savings arise from compression and are shared by all weighting rules
(Table~\ref{tab:deployment-cost}).

\subsection{Operator error under adaptive partitions}\label{sec:cpu}
We next evaluate discretization error under partitions whose cell contents change.
We evaluate 512 periodic signals under uniform, activity-driven, and variation-driven partitions
at budgets $N\in\{16,32,64,128\}$. Tokens average the same signal within each cell; a fixed
nonlinear score/value map defines attention, compared with a 4,096-cell numerical reference.
This isolates discretization error without training or benchmark labels.

At $N=32$, mean absolute error falls from 0.1524 to 0.0076 for activity-driven partitions and
from 0.1005 to 0.0020 for variation-driven partitions when measure weighting replaces count
weighting (Figure~\ref{fig:stability}). Uniform partitions give the same output under both rules.
All 6,144 field--partition--budget comparisons satisfy the coupled bound. A two-value construction attains the worst-case constant; the signal bounds are conservative.
As resolution increases, measure-weighted error decreases; the count-weighted errors remain near 0.15
and 0.10 on the two adaptive families. Appendices~\ref{app:cpu_generator}, \ref{app:cpu_results}, and~\ref{app:cpu_checks} give
the generator, outcomes, and numerical checks.

\begin{figure}[!htbp]
\centering\includegraphics[width=\textwidth]{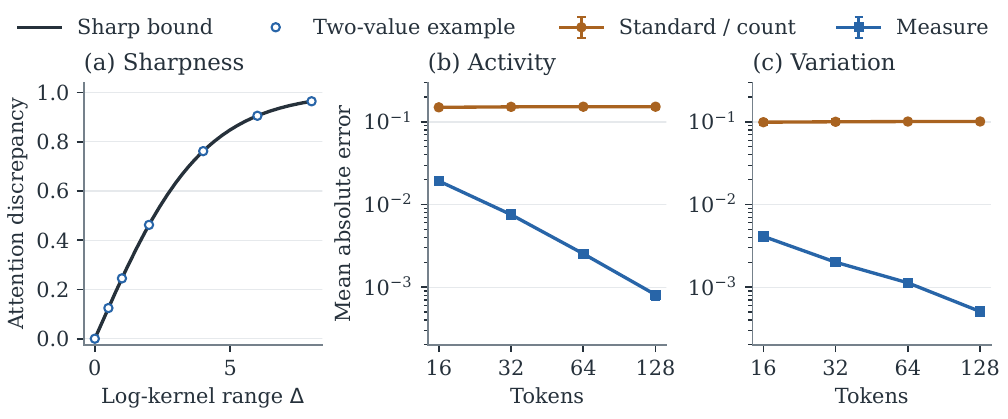}
\caption{\textbf{Exact worst-case bound and natural partition error.}
(a) A two-value construction attains $D_v\tanh(\Delta/4)$ with $D_v=1$ and no value perturbation.
(b,c) Mean absolute error against the dense reference over 512 periodic signals, with 95\% intervals
across signals. Standard and count weighting coincide. Adaptive cells change their contents
as they change width; measure weighting removes the count prior while retaining approximation error.}
\label{fig:stability}
\end{figure}

\subsection{Rate transfer and positional ablations}\label{sec:synthetic}
This study separates rate imbalance from positional representation in a controlled temporal task.
The task asks a four-layer, 1.77M-parameter model to predict the gap between events in
two modalities. Both arms use continuum RoPE and train at rate ratios one and two. Evaluation crosses four slow
rates with six ratios up to 32. Across ten paired seeds, 17 of 20 adjacent steps in the mean gain increase with ratio
(Figure~\ref{fig:factorial}). The seed-level trend has nominal $p=0.00195$, above the family threshold. An independent thirty-seed panel shows 19 of 20 ascents (nominal $p=9.3\times10^{-10}$) and a ratio-32 gain of
16.8\,pp with 95\% interval $[13.5,20.2]$
(Appendix~\ref{app:factorial_replication}). In the ten-seed panel, the equal-rate gain is $+0.09$\,pp
with 95\% CI $[-2.20,+2.38]$\,pp.

\begin{figure}[H]
\centering\includegraphics[width=\textwidth]{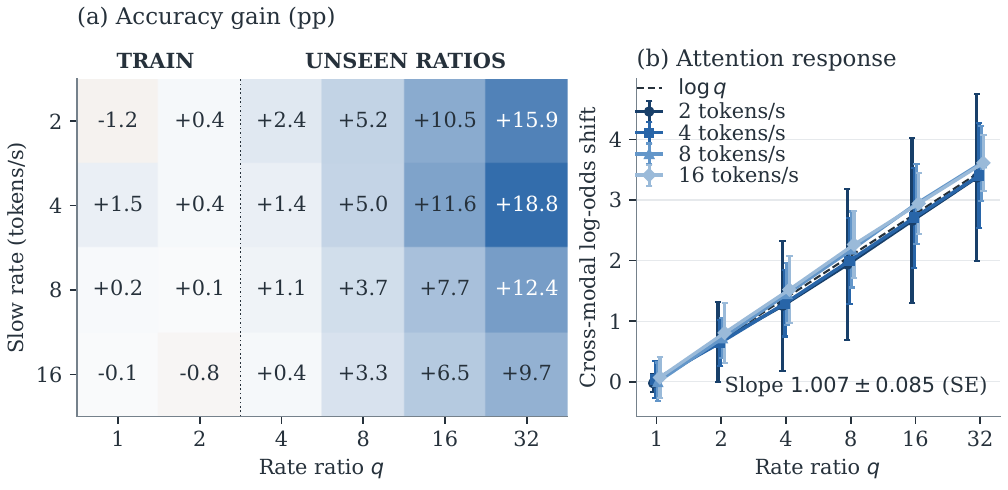}
\caption{\textbf{Transfer to unseen rate imbalance.} The task predicts the temporal gap between
events in two modalities. (a) Measure minus standard accuracy (pp), with continuum RoPE shared
by both arms, across four slow rates and six rate ratios; cells average ten paired seeds.
The dotted divider separates the two training ratios from unseen ratios. (b) Cross-modal
log-odds shift against rate ratio, with 95\% intervals over six seeds. The dashed line is the
fixed-score prediction $\log q$; trained arms may also differ in their scores.}
\label{fig:factorial}
\end{figure}

For token rates $r_{\rm fast}$ and $r_{\rm slow}$, let $q=r_{\rm fast}/r_{\rm slow}$.
The fixed-score prediction for the cross-modal log-odds shift is $\log q$. The measured slope
between trained arms is $1.007\pm0.085$ (standard error; six seeds). Sequence length follows from both rate variables. Appendices~\ref{app:synthetic_setup} and~\ref{app:factorial} give the setup and cellwise intervals.

\textbf{Position is a separate mechanism.} In the cross-rate component experiment, physical
coordinates raise held-out accuracy from approximately 20\% chance to 83.2\%; adding measure weighting
gives 83.4\% with no resolved increment. A shared physical clock is competitive.
In a separate eight-seed standard-attention comparison, models trained without support
attenuation score 11--21\,pp below models trained with width-tracking Gaussian attenuation at
merge factors two, four, and eight; Gaussian and box profiles have no resolved difference
(Appendices~\ref{app:transfer} and~\ref{app:merge}, Figures~\ref{fig:transfer} and~\ref{fig:merge}).

\section{Related work}\label{sec:related}
\textbf{Attention as weighted integration.} Token Merging adds the logarithm of each merged
patch's multiplicity to its score~\citep{bolya2023tome}. This is exactly measure weighting when
base patches have equal mass, including nonuniform groups. Continuum
Attention~\citep{calvello2025continuum}, neural operators~\citep{li2021fno,berner2026operators},
and Attention on the Sphere~\citep{bonev2025sphere} use function-space or quadrature formulations.
We characterize the finite split requirement and bound perturbations away from exact copies.

\textbf{Stability under representation changes.} Measure-based attention regularity and
Lipschitz analysis compare different sequence lengths and give token-count-independent bounds
on bounded feature domains~\citep{vuckovic2021regularity,castin2024smooth}.
Theorem~\ref{thm:stability} applies the sharp Hilbert-distance inequality of
\citet{cohen2023hilbert} to a physical mass coupling, separating value change from
reallocation through supportwise log-kernel oscillation, without an absolute normalizer bound. Head-wise correction addresses feature distortion left by token
merging~\citep{ichikawa2026hac}; our bound likewise retains this distortion after mass correction.
The composition result states the additional assumptions needed at network level.

\textbf{Position and compression.} Multidimensional rotary
coordinates~\citep{wang2024qwen2vl,ndrope2026} and Qwen2.5-Omni's temporal
alignment~\citep{xu2025qwen25omni} supply relevant physical-clock baselines; integrated Fourier
features~\citep{barron2021mipnerf} motivate averaging positional features over support.
Variable byte patches~\citep{pagnoni2024blt} and packing variable-resolution images into
fixed-length sequences~\citep{dehghani2023navit} illustrate distinct forms of variable token budgets.
Vision--language compression selects tokens by attention~\citep{chen2024fastv}, similarity
or clustering~\citep{tran2024pitome,yang2025visionzip,shang2025prumerge}, or hierarchical video
structure~\citep{li2025videochatflash}. These choices determine which features remain.
Refinement analysis asks how their inherited mass should enter aggregation and what error remains
when the features or their support change.

\section{Conclusion}\label{sec:conclusion}
Refinement symmetry separates represented information from token multiplicity. It characterizes
mass-proportional aggregation through finite invariance and bounds attention error when content and
support change. With matched content, position, and visibility, measure weighting preserves
predictions under duplication. At twofold merging, it improves all-partition-correct accuracy over
global count on both benchmarks, connecting the representation requirement to partition robustness.

\section*{Limitations}\label{sec:scope}
Exact invariance applies to copied content, position, and visibility; other retokenizations call
for perturbation analysis. Applying the depth bound to a pretrained decoder requires measuring
its blockwise constants. Deployment covers two benchmarks with frozen encoders and four partition
rules; extending to adaptive encoding requires measuring its feature changes. Robustness gains
depend on compression budget. Further pipeline savings require reducing the dominant encoder
and acquisition costs. Relative modality scales remain task-dependent modeling choices;
native-reference calibration preserves a checkpoint's prior at its native partition.
\label{page:endbody}

\clearpage
\section*{Reproducibility statement}
The appendices provide full proofs, implementation conventions, experimental protocols, and
statistical procedures. The source archive contains all figures and tables and the LaTeX files needed to compile
the manuscript. The ancillary supplement contains saved numerical results and figure data,
figure-rendering code, and deterministic summary checks, with source hashes linking the
reported observations.
\section*{AI use statement}
Generative AI assisted code development, manuscript writing, mathematical and reference checks, the authors reviewed and verified all AI-assisted work.
\bibliographystyle{iclr2027_conference}
\bibliography{refs}
\appendix
\raggedbottom
\section{Proofs and mathematical scope}\label{app:proofs}
\subsection{Refinement criterion}\label{app:criterion}
Fix a finite measure $\mu$. A partition represents a field through its cellwise values; two
representations are equivalent when their reconstructed fields agree almost everywhere. We allow
arbitrary measurable subdivisions with positive measure and discard null cells. If a restricted
family of tokenizers is used, it must admit a common refinement for each equivalent pair for the
necessity direction below to apply.

\begin{proposition}[Refinement criterion]\label{prop:classify}
For partition families closed under common refinement, a reduction is unchanged under every
content-preserving refinement if and only if it factors through the represented pair $(\mu,h_P)$.
A layer returning a field on its input partition is refinement-covariant if and only if its
reconstructed output depends only on $(\mu,h_P)$.
\end{proposition}

\begin{proof}[Proof of Proposition~\ref{prop:classify}]
Let $(P,h_P)$ and $(Q,h_Q)$ represent the same field. Their common refinement consists of the
positive-measure intersections $C\cap D$, $C\in P$, $D\in Q$. On each such intersection, the two
content values agree almost everywhere. A refinement-invariant reduction therefore gives the same
result on $P$, on the common refinement, and on $Q$. Its value defines a well-defined function of
$(\mu,h)$ on the set of representable fields. Conversely, refinement preserves $(\mu,h)$, so any
function of that pair is invariant.

For a layer, reconstruct its output as a field. Covariance says that the output on a refinement
restricts to the parent's output on every child. The two reconstructed outputs therefore agree on
the common refinement and hence almost everywhere on the domain. This defines the factored output
field. Conversely, if reconstructed outputs depend only on $(\mu,h)$, they agree before and after
refinement, which is precisely the stated covariance relation.
\end{proof}

The criterion concerns fields representable by the allowed partitions. Its layer statement concerns
operators whose outputs are constant on each input cell. For ordered sequences, cell-centroid features,
and changing causal horizons, the representation and refinement relation must also include that
structure.

If $F$ and $G$ are covariant layers on compatible representations, then
$F(G(RX))=F(RG(X))=RF(G(X))$. Residual sums and parallel branches follow by linearity of the
copying operation $R$. A content-only map $h_i\mapsto f(h_i)$ is covariant. A local map that also
reads $m_i$ need not be: $f(h_i,m_i)=m_i$ returns $m_i/k$ on equal children instead of the copied
parent value $m_i$.

\paragraph{Covariance does not select a density.}
Let $\nu$ be any fixed finite measure absolutely continuous with respect to $\mu$.
The reduction $\sum_i\nu(C_i)\varphi(h_i)$ is invariant because $\nu$ is additive and refinement
preserves content. For example, on $[0,1]$ with Lebesgue measure, take
$d\nu=(1+x^2)\,dx$. The ratios $\nu(C_i)/\mu(C_i)$ can differ between cells with the same
content. Thus covariant weights need not be a local function of mass and content alone.
Consistency with a \emph{specified} operator selects the intended density; Proposition~\ref{prop:split}
claims uniqueness only within its stated local class.

\subsection{Quadrature consistency and sampling-density bias}\label{app:quadrature}
\begin{equation}\label{eq:cont}
 \mathcal A_\mu^K(s,v)=\frac{\int K(s(x))v(x)\,d\mu(x)}{\int K(s(x))\,d\mu(x)}.
\end{equation}
Assume $\Om$ is a compact metric space, $\mu$ is finite and nonzero, $s,v$ are continuous, and
$K\circ s$ is continuous and strictly positive.
For a partition with maximum diameter $\delta$, write $f=K\circ s$ and $g=(K\circ s)v$. Uniform continuity
gives
\[
\left|\sum_i\mu(C_i)f(x_i)-\int f\,d\mu\right|
\le\mu(\Om)\,\omega_f(\delta),
\]
where $\omega_f$ is the modulus of continuity. The corresponding bound holds componentwise for
$g$. Both errors vanish with $\delta$, and $\int f\,d\mu>0$, so their ratio converges.
Merely assuming integrability would not guarantee convergence of arbitrary point evaluations.

\begin{proposition}[Sampling-density bias]\label{prop:incons}
Under the continuity and compactness assumptions above, if the empirical representative-point
measure converges weakly to $\rho\,d\mu/\int\rho\,d\mu$ for continuous $\rho>0$, then
\begin{equation}\label{eq:bias}
 \widehat{\mathcal A}_{\std}^K\longrightarrow\mathcal A_{\rho\mu}^K
 =\mathcal A_\mu^K+\frac{\Cov_\pi(v,\rho)}{\E_\pi[\rho]},
 \qquad d\pi=\frac{K(s)\,d\mu}{\int K(s)\,d\mu}.
\end{equation}
\end{proposition}

\begin{proof}[Proof of Proposition~\ref{prop:incons}]
Let $\eta_N=N^{-1}\sum_i\delta_{x_i}$. Weak convergence of $\eta_N$ and continuity of $f,g$ imply
\[
\frac{\int g\,d\eta_N}{\int f\,d\eta_N}
\longrightarrow
\frac{\int g\rho\,d\mu}{\int f\rho\,d\mu}.
\]
Using the probability measure $d\pi=f\,d\mu/\int f\,d\mu$, this limit equals
$\E_\pi[v\rho]/\E_\pi[\rho]$. Subtracting $\E_\pi[v]$ yields
$\Cov_\pi(v,\rho)/\E_\pi[\rho]$. The measure-weighted convergence follows from the preceding
quadrature bound.
\end{proof}

\paragraph{Second-order rates.}
On an interval with Lebesgue measure, the composite midpoint error for a twice continuously
 differentiable integrand is bounded by a constant times $\sum_i h_i^3$, where $h_i$ is cell width.
If $\max_i h_i=O(N^{-1})$, this sum is $O(N^{-2})$. For the standard estimator, suppose cells
are generated by a smooth increasing map of a uniform grid, with derivative bounded above and away
from zero. A physical cell midpoint differs from the image of its parameter midpoint by
$O(N^{-2})$. Applying midpoint quadrature in that parameter yields second-order convergence to the
sampling-density functional. These conditions explain the numerical rate in Figure~\ref{fig:convergence};
they are stronger than the assumptions needed for existence of the limit.

\paragraph{Numerical construction.}\label{app:numerics}
Figure~\ref{fig:convergence} uses $K(s)=e^s$, $\Om=[0,4]$ with Lebesgue measure, and
$s(x)=1.3\sin(2x)+0.4x$, $v(x)=\cos(1.7x)+0.5x$, and
$\rho(x)=1+0.8\sin(1.1x+0.3)$. Cell boundaries are the equal quantiles of the normalized
integral of $\rho$; representatives are physical cell midpoints. The token counts are
$16,64,256,1024,4096,16384$. Adaptive integration with absolute and relative tolerances
$2\times10^{-13}$ gives $\mathcal A_\mu=1.8453563272$ and
$\mathcal A_{\rho\mu}=0.9459689468$. Their difference is approximately 0.8994.

\begin{figure}[!htbp]
\centering\includegraphics[width=\textwidth]{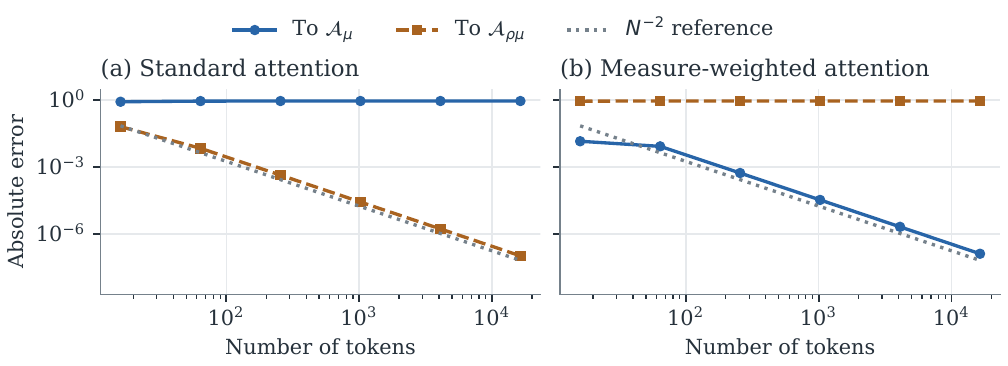}
\caption{\textbf{Convergence to different limits.} Each estimator is compared with the intended
functional $\mathcal A_\mu$ and the sampling-density functional $\mathcal A_{\rho\mu}$.
The same smooth signal is evaluated on increasingly fine nonuniform partitions. Reference integrals
use adaptive quadrature. The dashed reference has slope $-2$; the nonvanishing error reflects the
difference between the two limits. The construction above specifies the signal and partitions.}
\label{fig:convergence}
\end{figure}

\subsection{Splitting, uniqueness, and masks}\label{app:splitting}
\begin{proof}[Proof of Proposition~\ref{prop:split}]
For a token with score $s$ and value $v$, children of masses $m_j$ contribute
$\sum_j m_j K(s)v=m K(s)v$ to the numerator and $\sum_j m_j K(s)=m K(s)$ to the denominator.
All other terms remain unchanged, proving sufficiency for arbitrary mass splits.

For necessity in the local class, choose a second token with a different value and positive
weight. Preserving the two-token output for all such choices requires the total weight assigned
to the children to equal the parent's weight. An equal split thus gives
$k w(m/k)=w(m)$ for every integer $k\ge1$. It follows that $w(qm_0)=q w(m_0)$ for every positive
rational $q$. Set $c=w(m_0)/m_0$. For arbitrary $x>0$, choose rational sequences
$q_n\uparrow x/m_0$ and $r_n\downarrow x/m_0$. Monotonicity gives
$cq_n m_0\le w(x)\le cr_n m_0$, so $w(x)=cx$. Conversely, this form preserves all splits by
additivity, and its common factor cancels in the normalized attention.
\end{proof}

For the exponent family, the children contribute
$k(m/k)^\alpha K(s)=k^{1-\alpha}m^\alpha K(s)$. Requiring equality for $k>1$ gives $\alpha=1$.
A single split factor does not characterize all monotone weights. For sufficiently small
$\epsilon>0$, $w(m)=m[1+\epsilon\sin(2\pi\log_2 m)]$ is increasing and invariant to halving,
but not to all integer splits.

Consistency is weaker than exact finite-partition invariance. For example,
$w(m)=m\exp(c m^2)$ differs from $m$ by $O(m^3)$. Under regular refinement it can retain a
second-order quadrature error while failing exact split invariance. This separates the two uses
of token measure in the main text.

The attention proof applies row by row to visible keys. For network-level duplication, each copy
must receive the same incoming representation, positional treatment, and visible total mass.
A parent that sees itself as a whole is not equivalent to a causal chain in which child $j$ sees
only children $1,\ldots,j$. A physical causal operator may still be well-defined, but its outputs
at distinct child times are not copies of a single parent output.

\subsection{Learning a measure exponent}\label{app:identification}
\paragraph{Gradient identity.}
For one row, let $g_i=\partial L/\partial z_i$. Softmax shift invariance implies $\sum_i g_i=0$.
The chain rule gives
\[
\frac{\partial L}{\partial\alpha}=\sum_i g_i b_i
=\sum_i g_i(b_i-\bar b)=n\Cov_i(g_i,b_i).
\]
Cauchy--Schwarz then yields
$\bigl|\partial L/\partial\alpha\bigr|\le\sqrt n\,\|\nabla_z L\|_2\,\sigma(b)$, since
$\sum_i(b_i-\bar b)^2=n\sigma^2(b)$. The covariance uses the uniform empirical distribution over
keys. For a shared exponent in multiple rows, heads, or layers, the total gradient is the sum
of these row contributions; cancellation can further reduce it.

If a fraction $p$ of keys has physical support, its log-measure has within-group variance
$\sigma_{\mathrm{phys}}^2$ and mean $\ell$, and all other keys have $b=0$, then
\begin{equation}\label{eq:channels}
\sigma^2(b)=p\sigma_{\mathrm{phys}}^2+p(1-p)\ell^2.
\end{equation}
This is the law of total variance. It distinguishes variation among physical keys from their
common offset relative to the auxiliary keys. Neither component alone determines the loss
sensitivity, which also depends on the gradient direction.

\begin{proposition}[Ambiguity with shared type biases]\label{prop:identification}
Fix score arrays $s^{(r)}$ over a finite family of rows or tokenizations $r\in\mathcal R$.
Let $\tau(i)$ index a key's modality or auxiliary type, and consider
$z_i^{(r)}=s_i^{(r)}+c_{\tau(i)}+\alpha b_i^{(r)}$. Define
\begin{equation}\label{eq:did}
D_{\mathrm{id}}^2=\min_{d\in\R^T}\frac1{|\mathcal R|}\sum_{r\in\mathcal R}
\Var_i\big(b_i^{(r)}-d_{\tau(i)}\big).
\end{equation}
There is a shared vector $d$ such that $(\alpha,c)$ and $(\alpha+t,c-td)$ give the same attention
probabilities on every row for all $t$ if and only if $D_{\mathrm{id}}=0$.
\end{proposition}

\begin{proof}
The logit change is $t(b_i^{(r)}-d_{\tau(i)})$. Two finite logit vectors have identical softmax
probabilities exactly when they differ by a constant within a row. Hence equality for all $t$
and all $r$ is equivalent to every variance in Equation~\eqref{eq:did} vanishing. These are
nonnegative terms, so this is equivalent to a zero minimum.
\end{proof}

The least-squares minimum exists; its vector need not be unique because common shifts cancel.
Taking $d=0$ bounds $D_{\mathrm{id}}^2$ by the mean rowwise variance of $b$.
Adding a fixed type-specific reference offset to $b$ leaves $D_{\mathrm{id}}$ unchanged. A single
configuration with constant $b$ within each type has $D_{\mathrm{id}}=0$; several configurations
need not, because one shared bias vector must explain them all. Positive $D_{\mathrm{id}}$
removes this particular ambiguity in attention probabilities. It does not establish identification
through arbitrary values, downstream losses, or a jointly trainable score function.

\paragraph{A uniform physical offset.}
When $b_i=\ell$ on physical keys and zero elsewhere, all probabilities depend on $\alpha$ and
$\ell$ only through $u=\alpha\ell$. If their baseline attention share is $p_0$, their new share is
$w(u)=p_0e^u/(1-p_0+p_0e^u)$, with $w'(u)=w(u)[1-w(u)]$.
For fixed within-group value averages $\bar v_P,\bar v_A$,
\[
\frac{\partial L}{\partial\alpha}
=\ell w(u)[1-w(u)]\left\langle\nabla_o L,\bar v_P-\bar v_A\right\rangle.
\]
Thus a large log-measure dispersion can coincide with a small gradient through saturation or
alignment. If the profile in $u$ has a unique finite minimizer, $\alpha^*\ell=u^*$; otherwise that
argmin relation is not defined. The baseline share $p_0$ is an attention share, not generally the
fraction of physical tokens.

\section{Coupled stability, support, and composition}\label{app:stability}
\subsection{A sharp change-of-weight bound}\label{app:tilt}
We use total variation with the convention
$\operatorname{TV}(P,Q)=\frac12\sum_a|P_a-Q_a|$.
The following estimate is a sharp relation between total variation and Hilbert
distance~\citep{cohen2023hilbert}; we include its finite-space proof to fix the normalization.

\begin{lemma}[Bounded likelihood-ratio range]\label{lem:tilt}
Let $P$ be a probability distribution on a finite set and let $f>0$ on its support.
Set $Q_a=P_af_a/\E_P f$ and
$\Delta=\max_a\log f_a-\min_a\log f_a$. Then
$\operatorname{TV}(P,Q)\le\tanh(\Delta/4)$.
\end{lemma}
\begin{proof}
Write $a=\min f$, $b=\max f$, and $u=\E_Pf$. If $a=b$, then $P=Q$.
Otherwise, the convex function $x\mapsto|x-u|$ lies below the chord joining its values at $a,b$.
Taking expectations gives
\[
 \operatorname{TV}(P,Q)
 =\frac{\E_P|f-u|}{2u}
 \le\frac{(b-u)(u-a)}{u(b-a)}
 \le\frac{\sqrt b-\sqrt a}{\sqrt b+\sqrt a}
 =\tanh(\Delta/4).
\]
The middle expression is maximized at $u=\sqrt{ab}$, as follows by differentiating
$(a+b-u-ab/u)/(b-a)$ on $[a,b]$.
\end{proof}

\begin{proof}[Proof of Theorem~\ref{thm:stability}]
Lift both attention distributions to the support of the coupling:
\[
 P_{ij}=\frac{\gamma_{ij}K(s_i)}{\sum_{ab}\gamma_{ab}K(s_a)},\qquad
 Q_{ij}=\frac{\gamma_{ij}K(t_j)}{\sum_{ab}\gamma_{ab}K(t_b)}.
\]
The first marginal of $P$ is the original measure-weighted attention on $i$; the second marginal of
$Q$ is the new attention on $j$. Their density ratio is the normalized tilt
$f_{ij}=K(t_j)/K(s_i)$, so Lemma~\ref{lem:tilt} applies with the stated $\Delta$.
For any values of diameter $D_v$, the positive and negative parts of $P-Q$ have equal mass
$\operatorname{TV}(P,Q)$. Their normalized value means lie in the convex hull of $\{v_i\}$,
whose diameter is at most $D_v$ in any norm. Therefore
\[
\begin{aligned}
 \|B-A\|
 &\le\|\E_Qv_i-\E_Pv_i\|+\E_Q\|w_j-v_i\|\\
 &\le D_v\operatorname{TV}(P,Q)+\epsilon_v
 \le D_v\tanh(\Delta/4)+\epsilon_v.
\end{aligned}
\]
The same proof permits replacing $\epsilon_v$ by its $Q$-weighted average. Interchanging the two
representations gives a bound using the diameter of $\{w_j\}$ instead.

Sharpness already holds for the exponential kernel and two scalar values $0,1$.
For any $\Delta\ge0$, use a diagonal mass coupling and choose baseline attention probabilities
\[
 P=\left(\frac{e^{\Delta/2}}{1+e^{\Delta/2}},
          \frac{1}{1+e^{\Delta/2}}\right).
\]
Score changes $-\Delta/2,+\Delta/2$ swap these probabilities. The scalar output difference is
exactly $\tanh(\Delta/4)$, with $\epsilon_v=0$ and $D_v=1$.
The coefficient of $\epsilon_v$ is attained by translating all values by the same vector while
leaving scores unchanged. Finally, $|\log K(t_j)-\log K(s_i)|\le\epsilon_s$ implies
$\Delta\le2\epsilon_s$.
\end{proof}

The theorem requires finite positive kernel values, but no lower bound on individual token mass
or on the absolute score normalizer. A common multiplicative change of the kernel in either
representation cancels. Equal total masses allow the physical intersection coupling used in the
main text; algebraically, any coupling of the two separately normalized mass vectors suffices.

\paragraph{Partitions and continuum comparison.}
For partitions of one domain, the row and column sums of
$\gamma_{ij}=\mu(C_i\cap D_j)/\mu(\Om)$ are the required normalized masses.
If representatives $x_i,y_j$ belong to intersecting cells of diameters at most $h,h'$, then
$d(x_i,y_j)\le h+h'$. Thus Lipschitz value and log-kernel fields have
$\epsilon_v\le L_v(h+h')$ and $\Delta\le2L_s(h+h')$.
This applies to nested and non-nested partitions.
The same proof works on a general probability space with essential bounds and integrals in
place of maxima and sums, provided both kernel normalizers are finite and positive and the
values are integrable under the corresponding lifted attention laws. Coupling a token $i$ to physical points in $C_i$ directly compares its
quadrature with a continuous target: value error at most $L_vh$ and log-kernel error at most
$L_sh$ yield discrepancy at most $L_vh+D_v\tanh(L_sh/2)$, where $D_v$ bounds the diameter of
the reference values. Encoder approximation errors add to these paired errors.

\subsection{Support geometry and rotary scores}\label{app:support_bound}
For $U(p)=\mathbb E_{X\sim p}R(X)$, the support perturbation bound is
\begin{equation}\label{eq:support_stability}
 \|U(p)-U(p')\|_{\mathrm{op}}\le\min\{2,\omega_{\max}W_1(p,p')\}.
\end{equation}
In this subsection, vector norms are Euclidean and operator norms are their induced spectral norms.
Let $R(x)$ be a real block-diagonal rotary matrix with two-dimensional rotation blocks of angular
frequencies $\omega_r$, and write $\omega_{\max}=\max_r|\omega_r|$.
For probability measures with finite first moment, define $U(p)=\int R(x)\,dp(x)$.
Each $R(x)$ is orthogonal, so $\|U(p)\|_{\mathrm{op}}\le1$.
For a single block, the operator norm of the difference is
$2|\sin(\omega_r(x-y)/2)|$; hence
\[
 \|R(x)-R(y)\|_{\mathrm{op}}
 \le\min\{2,\omega_{\max}|x-y|\}.
\]
For any coupling $\lambda$ of $p,p'$, Jensen's inequality and this estimate give
\[
 \|U(p)-U(p')\|_{\mathrm{op}}
 \le\int\|R(x)-R(y)\|_{\mathrm{op}}\,d\lambda(x,y)
 \le\omega_{\max}\int|x-y|\,d\lambda(x,y).
\]
Infimizing over couplings, and also using $\|U(p)-U(p')\|_{\mathrm{op}}\le2$, proves
Equation~\eqref{eq:support_stability}. In multiple physical dimensions, the same proof uses
frequency vectors and their largest Euclidean norm, together with Euclidean transport distance.

\paragraph{A score bound with explicit constants.}
Consider a head with score
\[
 s=\frac{q^\top U(p_q)^\top U(p_k)k}{\sqrt d}
\]
and a corresponding primed configuration. Suppose both query norms are at most $Q$, both key
norms at most $H_k$, and
$\|q-q'\|\le\epsilon_q$, $\|k-k'\|\le\epsilon_k$.
Let $W_1(p_q,p_q')\le\delta_q$ and $W_1(p_k,p_k')\le\delta_k$.
Add and subtract terms changing one factor at a time. Contraction of each averaged rotation yields
\begin{equation}\label{eq:score_support_bound}
 |s-s'|\le
 \frac{H_k\epsilon_q+Q\epsilon_k+
 QH_k[\min\{2,\omega_{\max}\delta_q\}+
       \min\{2,\omega_{\max}\delta_k\}]}{\sqrt d}.
\end{equation}
For exponential attention this is an $\epsilon_s$ in Theorem~\ref{thm:stability}; for another
positive kernel, multiply by a Lipschitz constant of $\log K$ on the attained score range.
A fixed readout query has $\epsilon_q=\delta_q=0$.
If the supports lie in intersecting cells, every pair of support points is at distance at most
$h+h'$, so $W_1\le h+h'$. Equal mass without a constraint on support displacement gives no such bound.
For Gaussian extent filters, whose supports are unbounded, coupling
$x+\sigma Z$ and $x'+\sigma'Z$ with the same standard normal $Z$ instead gives
$W_1\le|x-x'|+\sqrt{2/\pi}|\sigma-\sigma'|$; the variance-matched filter uses
$\sigma=h/\sqrt{12}$.

\paragraph{Independent asymmetric supports.}
For independent $X_i,X_j$, Fubini's theorem gives
\[
 \E e^{\mathrm i\omega(X_j-X_i)}
 =\E e^{-\mathrm i\omega X_i}\E e^{\mathrm i\omega X_j}
 =e^{\mathrm i\omega(x_j-x_i)}\overline{\phi_i(\omega)}\phi_j(\omega).
\]
No symmetry assumption is required. Averaging the real rotary bilinear form similarly gives
$q^\top U(p_i)^\top U(p_j)k$ when content vectors are fixed over the supports and the two support
draws are independent. This equality concerns the score, not the exponential of that score or
the normalized attention output. Uniform centered intervals and centered Gaussians have real
characteristic functions; an asymmetric distribution generally does not.

\subsection{Alignment maps and network propagation}\label{app:network_bound}
\begin{corollary}[Composition across layers]\label{cor:network}
Let $R_\ell$ align coarse and fine representations at layer $\ell$. Suppose the fine block has
Lipschitz constant $L_\ell$, and its local discrepancy from the aligned coarse block is at most
$\eta_\ell$ on the relevant states. If the initial discrepancy is $e_0$, then
\begin{equation}\label{eq:network}
 e_L\le e_0\prod_{r=1}^L L_r+
 \sum_{\ell=1}^L\eta_\ell\prod_{r=\ell+1}^L L_r.
\end{equation}
\end{corollary}

Let $X_\ell,Y_\ell$ be normed state spaces for the two representations at layer $\ell$.
Write their blocks as $F_\ell:X_{\ell-1}\to X_\ell$ and
$G_\ell:Y_{\ell-1}\to Y_\ell$, and choose alignment maps $R_\ell:X_\ell\to Y_\ell$.
The maps need not be linear. Define the local alignment discrepancy by
\begin{equation}\label{eq:commutator_bound}
 \|G_\ell(R_{\ell-1}x)-R_\ell(F_\ell x)\|_{Y_\ell}\le\eta_\ell
\end{equation}
for the relevant coarse states $x$. Assume $G_\ell$ is $L_\ell$-Lipschitz for the fine states and
aligned coarse states being compared. For trajectories $x_\ell=F_\ell x_{\ell-1}$,
$y_\ell=G_\ell y_{\ell-1}$, set
$e_\ell=\|y_\ell-R_\ell x_\ell\|_{Y_\ell}$.

\begin{proof}[Proof of Corollary~\ref{cor:network}]
Insert $G_\ell(R_{\ell-1}x_{\ell-1})$ and apply the triangle inequality:
\[
\begin{aligned}
 e_\ell
 &\le\|G_\ell(y_{\ell-1})-G_\ell(R_{\ell-1}x_{\ell-1})\|_{Y_\ell}
      +\eta_\ell\\
 &\le L_\ell e_{\ell-1}+\eta_\ell.
\end{aligned}
\]
Repeated substitution proves Equation~\eqref{eq:network}, with empty products equal to one.
\end{proof}

The discrepancy in Equation~\eqref{eq:commutator_bound} compares complete blocks, including changes
in masks, normalization, routing, and residual branches. A Lipschitz readout supplies a corresponding
output bound.
For logits $z,z'$ with $\|z-z'\|_\infty\le\epsilon$, if the reference top logit exceeds every
other logit by more than $2\epsilon$, then its top index is unchanged: each pairwise margin
decreases by at most $2\epsilon$.

\paragraph{Non-nested partitions.}
For non-nested partitions, the coupling supplies a common index set. Let
$I=\{(i,j):\gamma_{ij}>0\}$ and lift both states to this common index set,
$\widehat h_{ij}=h_i$, $\widehat h'_{ij}=h'_j$, with norm
$\|u\|=\max_{(i,j)\in I}\|u_{ij}\|$ and masses $\gamma_{ij}$.
For blocks preserving identical splits, the original coarse computation is represented exactly
by the lifted block with coarse positional and visibility metadata copied to its duplicates;
the same holds for the fine block with fine metadata. These two lifted blocks act on the same
ambient state space, so the alignment in the corollary can be the identity. This construction
requires every block to preserve identical splits with the stated positional and visibility metadata.
The original partition coupling remains the basis for comparing the lifted states.

\paragraph{Explicit attention-block constants.}
For a quantitative sufficient condition, fix a mass coupling and set
$e=\max_{\gamma_{ij}>0}\|h_i-h_j'\|$. For each matched query pair, assume its visible key masses
admit a coupling supported on these same matched key pairs.
Suppose paired values differ by at most $v_\ell e+\tau_\ell$, paired log-kernels by at most
$c_\ell e+\zeta_\ell$, and the reference value diameter is at most $D_\ell$.
For a shared update $f_\ell(h,A)$ satisfying
\[
 \|f_\ell(h,A)-f_\ell(h',A')\|
 \le a_\ell\|h-h'\|+b_\ell\|A-A'\|,
\]
Theorem~\ref{thm:stability} gives the recurrence
\begin{equation}\label{eq:nonlinear_network}
 e_\ell\le a_\ell e_{\ell-1}+b_\ell\left[
 v_\ell e_{\ell-1}+\tau_\ell+
 D_\ell\tanh\!\left(\frac{c_\ell e_{\ell-1}+\zeta_\ell}{2}\right)\right].
\end{equation}
Since $\tanh u\le u$ for $u\ge0$, one can use
$\widehat L_\ell=a_\ell+b_\ell(v_\ell+D_\ell c_\ell/2)$ and
$\widehat\eta_\ell=b_\ell(\tau_\ell+D_\ell\zeta_\ell/2)$ in the same product-sum recurrence.
These constants apply to the coupled-error recurrence above.

With Euclidean vector norms and spectral operator norms, for exponential dot-product attention
with $q=W_Qh$, $k=W_Kh$, $v=W_Vh$ and both state sets
bounded by $\|h\|\le H$, valid choices are
\[
 v_\ell=\|W_V\|_{\mathrm{op}},\qquad
 c_\ell=\frac{2\|W_Q\|_{\mathrm{op}}\|W_K\|_{\mathrm{op}}H}{\sqrt d},\qquad
 D_\ell\le2\|W_V\|_{\mathrm{op}}H.
\]
For averaged rotary scores these bounds remain valid for content perturbations; the support
terms of Equation~\eqref{eq:score_support_bound} enter $\zeta_\ell$.
For multiple heads in the Euclidean norm, combine head discrepancies by their squared sum before
applying the output projection. Shared pointwise maps, residual updates, and featurewise
normalizations contribute their Lipschitz constants on the stated state set; regularization is
needed for normalization at zero variance. These constants contain no token-count factor, but
their depth product can be large. Even the pointwise map $h\mapsto2h$ amplifies an initial error
by $2^L$ after $L$ layers.

\subsection{Modality scales and reference calibration}\label{app:calibration}
\begin{proposition}[Local scales and observable priors]\label{prop:calibration}
Within each modality, positive local weighting invariant to every binary mass split for every score/value configuration has the
form $w_\tau(m)=\kappa_\tau m$. A typewise rescaling preserves all attention probabilities exactly when it is common to every
co-visible key type, including auxiliary atoms, and hence constant on each connected component
of their co-visibility graph.
\end{proposition}

\begin{proof}[Proof of Proposition~\ref{prop:calibration}]
Fix a modality and a mass $m=a+b$. Assign the split token scalar value one and an unsplit
competitor value zero, with positive kernel values and positive competitor weight $c$.
The map $u\mapsto u/(u+c)$ is injective, so preservation of this attention output forces
$w(a+b)=w(a)+w(b)$. Positivity implies strict monotonicity because
$w(y)-w(x)=w(y-x)>0$ for $y>x$. Additivity gives rational homogeneity, and monotone rational
approximation yields $w(m)=\kappa m$ for $\kappa>0$. The argument applies separately to each
modality; conversely such weights preserve every split by additivity.

For the second claim, let $p_i>0$ be the attention probabilities on the visible keys of a row,
and rescale type $\tau$ by $a_\tau>0$. The new probabilities are
$p_i'=a_{\tau(i)}p_i/\sum_j a_{\tau(j)}p_j$.
Equality $p_i'=p_i$ for every visible $i$ is equivalent to all visible factors being the same.
Consequently co-visible types have equal factors, and equality propagates along every edge of
the co-visibility graph. Conversely factors constant on each connected component are common
within every row and cancel in its normalizer.
\end{proof}

The graph includes every key type whose scale is allowed to vary. Auxiliary atoms with a fixed
unit weight can instead be included as vertices with fixed factor one; any connected component
containing them then has its common factor fixed. Strict positivity is required on visible keys;
masked keys are omitted. The statement concerns attention probabilities, so accidental output
equality from identical values does not establish scale invariance.
With uniform native measures $m_\tau^0$, choosing the relative factors $m_i/m_\tau^0$ gives one
on every native physical token and leaves auxiliary factors at one. This preserves the native
attention formula exactly. It selects a reference prior; changing relative modality scales
changes the operator. Algebraic identity does not require two floating-point execution paths
to be bitwise identical.

\subsection{Physical visibility and overlapping representations}\label{app:coverage}
Fix a physical query $q$ and a measurable region $V(q)$ with positive finite measure.
For piecewise-constant key scores and values, the numerator of the intended visible operator is
\[
 \int_{V(q)} K(s(x))v(x)\,d\mu(x)
 =\sum_i\mu(C_i\cap V(q))K(s_i)v_i.
\]
Its denominator is the same identity without values. Under a subdivision of $C_i$ with copied
scores and values, the child visible masses sum to the parent's visible mass, including when
the boundary crosses the cell. This proves invariance for the same physical query. Zero visible
masses are masked. A binary mask testing only a cell representative is generally different from
this partial-cell integral. Support-sensitive scores that change on subdivision require their
own perturbation term even when visible masses are additive.

For two partitions and the same $V(q)$, the visible coupling is
$\gamma_{ij}(q)=\mu(C_i\cap D_j\cap V(q))/\mu(V(q))$.
Its support is contained in that of the full physical coupling, so a bound on all paired hidden
states also bounds the pairs visible to this query. This gives the common visible-key coupling
required by Equation~\eqref{eq:nonlinear_network}. Matched queries at different physical times
need not share $V(q)$; a changed causal prefix is an additional perturbation.

To see the boundary contribution separately, suppose two regions $V,V'$ use identical positive
kernel field $g$ and values of diameter $D$. Their normalized visible distributions satisfy
\[
 \operatorname{TV}(P_V,P_{V'})
 =1-\frac{\int_{V\cap V'}g\,d\mu}
            {\max\{\int_Vg\,d\mu,\int_{V'}g\,d\mu\}}.
\]
Indeed their common mass is the integral of the smaller density on the intersection.
Multiplying this expression by $D$ bounds the output change. Small boundary displacement alone
need not make this term small without control of the affected kernel mass and the normalizers.
Different kernels or values introduce the corresponding additional perturbations.

\paragraph{A partition of unity for overlapping supports.}
Let $\psi_i\ge0$ satisfy $\sum_i\psi_i=1$ almost everywhere. Setting
$m_i(q)=\int_{V(q)}\psi_i\,d\mu$ gives the exact numerator
\[
 \sum_i m_i(q)K(s_i)v_i
 =\int_{V(q)}\sum_i\psi_i(x)K(s_i)v_i\,d\mu(x),
\]
and the corresponding denominator. Replacing $\psi_i$ by nonnegative functions summing to it,
with copied scores and values, preserves both integrals. For two such covers $\psi_i,\chi_j$,
\[
 \gamma_{ij}(q)=\frac{\int_{V(q)}\psi_i(x)\chi_j(x)\,d\mu(x)}{\mu(V(q))}
\]
has the required marginals because each family sums to one. Theorem~\ref{thm:stability} therefore
applies with these effective masses after omitting zero-mass tokens.

This operator mixes token kernels and values, rather than applying nonlinear attention to an
averaged hidden field. On a unit domain, take two fully overlapping tokens with
$\psi_1=\psi_2=1/2$, $h_1=-1,h_2=1$, and score and value both equal to $h$.
Their mean hidden field is zero, but the exponential-kernel mixture returns $\tanh(1)$;
one token with value and score zero returns zero. Coverage accounting does not remove this
feature disagreement. The construction needs an explicit partition of unity, which is not
supplied merely by assigning full physical area to every overlapping crop.

\section{Implementation and primitive checks}\label{app:implementation}
\subsection{Measures, reference scales, and masks}\label{app:reference_impl}
For a temporal cell of duration $h_i$, the physical measure is $h_i$. For a video patch, a
spatiotemporal measure may also include its normalized image area. The experiments specify which
of these factors changes: temporal sweeps hold spatial patching fixed; the spatial-resolution
probe is described separately in Appendix~\ref{app:multiscale}. A text token may have a discrete
reference measure, but the physical interventions in this study do not retokenize text.

The two implementations use related but distinct calibration conventions. The synthetic model
expresses supported-token durations relative to a fixed scalar calibration obtained from its
training configuration; auxiliary slots retain unit weight. Released-model native-rate sweeps use
per-modality reference cells so that their intended native bias is zero. Split experiments preserve
the reference weights and divide only the duplicated mass. The calibration is part of each
experimental specification; it is not recomputed after the intervention.

A minimal attention implementation is
\begin{verbatim}
scores = q @ k.transpose(-2, -1) / sqrt(head_dim)
scores = transform_scores(scores)   # optional temperature/softcap
scores = scores + log_reference_mass[..., None, :]
scores = scores.masked_fill(~visible, -inf)
output = softmax(scores, dim=-1) @ v
\end{verbatim}
The score transformation acts before the log-measure term. Applying a temperature to
$s+\log m$ changes the exponent on the measure, while softcapping~\citep{team2024gemma2} that sum prevents children from
contributing additive mass. An equivalent implementation multiplies the unnormalized transformed
score weights by $m$ and renormalizes. Masked keys have zero attention weight and do not participate
in physical calibration. Every query must retain at least one visible key.

An auxiliary key with zero \emph{additive bias} has multiplicative weight one. A token with truly
zero integration measure would instead have zero weight and be excluded. Prompt and readout weights
are modeling choices fixed by reference calibration.
Similarly, a per-modality reference shift that leaves within-modality odds unchanged can still
change competition with text keys; cross-modal and whole-row calibration are different operations.

\subsection{Support-integrated position}\label{app:position}
The synthetic implementation allocates a geometric ladder of physical wavelengths to rotary pairs.
For a support distribution $p_i$ centered at $x_i$, define its centered characteristic function
$\phi_i(\omega)=\mathbb E_{X_i\sim p_i}e^{\mathrm i\omega(X_i-x_i)}$. Independent query and key
supports then satisfy
\begin{equation}\label{eq:char}
\E_{y\sim p_i,x\sim p_j}[e^{\mathrm i\omega(x-y)}]
=e^{\mathrm i\omega(x_j-x_i)}\overline{\phi_i(\omega)}\phi_j(\omega).
\end{equation}
The conjugate also covers asymmetric supports.
Let $R_\omega(x)$ denote a rotary pair's rotation through angle $\omega x$.
For symmetric supports, $\phi_i$ is real, and averaging these rotations while holding content fixed
within each support gives
$\widetilde q_i=\phi_i(\omega)R_\omega(x_i)q_i$ and
$\widetilde k_j=\phi_j(\omega)R_\omega(x_j)k_j$. Their inner product is
$\phi_i(\omega)\phi_j(\omega)q_i^\top R_\omega(x_j-x_i)k_j$.
A simultaneous shift of all coordinates leaves this expression unchanged in exact arithmetic.
Carrying coordinates relative to a local origin reduces phase precision loss from large absolute timestamps.

For a uniform interval of width $h$,
$\phi(\omega)=\sinc(\omega h/2)$, with $\sinc z=\sin z/z$ and $\sinc0=1$;
a variance-matched Gaussian gives $e^{-\omega^2h^2/24}$.
If the query is an evaluation point, only the key support is integrated. If both tokens summarize
cells, the two-sided factor is appropriate for the averaged positional logit. These are different
operators. In particular, integration cannot be moved through softmax, so neither expression alone
establishes exact geometric refinement covariance of an attention layer.

Writing $z=\omega h$, the box and Gaussian profiles differ first at fourth order:
\[
\exp(-z^2/24)-\sinc(z/2)=z^4/2880+O(z^6).
\]
For a Gaussian, the half-gain point is $\sqrt{24\log2}$ in the key-only form and
$\sqrt{12\log2}$ in the two-sided form at equal extents. The factor of $\sqrt2$ follows from
which supports are integrated, not from the choice of reference mass. The support-free readout query
has unit positional gain, but intermediate physical queries still use the two-sided implementation.
An end-to-end comparison of the two forms must also change the intermediate query factors.

For multidimensional rectangular supports, the characteristic function is a product of the
axiswise factors. For a Gaussian with covariance $\Sigma_i$, it is
$\exp(-\tfrac12\omega^\top\Sigma_i\omega)$. Equal-volume shapes can have different covariance
and different rotary gains. Overlapping receptive fields and multiscale views require an explicit
support model; their measure cannot be inferred from patch count without that choice.

\begin{figure}[!htbp]
\centering\includegraphics[width=\textwidth]{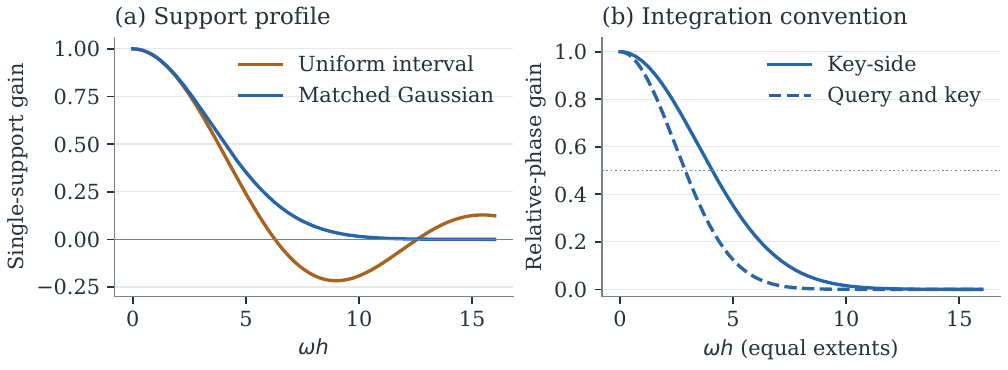}
\caption{\textbf{Two-sided integration squares the support factor and narrows its frequency
response.} A uniform cell gives a sinc factor;
a variance-matched Gaussian gives a smooth nonnegative factor. At equal query/key extents,
the squared single-support Gaussian is the two-sided convention used throughout.
These are analytic response curves.}
\label{fig:position}
\end{figure}

\subsection{Primitive checks and their scope}\label{app:primitives}
The local suite evaluates 34 configurations in float64 under a content-preserving duplicate split.
Its largest absolute output deviation is compared with the appropriate copied or invariant reference.
Table~\ref{tab:primitive_numeric} shows the principal operations. The entries are deterministic
constructed examples; the algebraic argument and the precise operation define the scope of each row.

\begin{table}[!htbp]
\centering\small
\caption{\textbf{Measure-aware operations preserve the tested duplicate splits.} Representative duplicate-refinement checks. Deviations are absolute and use each
operation's own units; their magnitudes should not be compared across rows. The measure-aware
capacity result concerns the constructed admission instance described in the text.}
\label{tab:primitive_numeric}
\begin{tabular}{@{}lrr@{}}
\toprule
Operation & Count form & Measure-aware form \\
\midrule
Attention & $4.7\times10^{-1}$ & $2.2\times10^{-16}$ \\
Mean pooling & $5.5\times10^{-1}$ & $1.1\times10^{-16}$ \\
Per-token loss & $8.4\times10^{-1}$ & $8.9\times10^{-16}$ \\
Token-axis normalization & $4.4\times10^{-1}$ & $6.7\times10^{-16}$ \\
MoE load-balancing loss & $1.3\times10^{-1}$ & $6.7\times10^{-16}$ \\
Expert capacity & $4.4$ & $0$ \\
Log-length score scaling & $1.1\times10^{-1}$ & $4.4\times10^{-16}$ \\
Length-dependent temperature & $7.9\times10^{-2}$ & $4.4\times10^{-16}$ \\
\bottomrule
\end{tabular}
\end{table}

\paragraph{Aggregations.}
A measure-weighted mean is $\sum_i m_i f_i/\sum_i m_i$. Applying it to feature moments gives
refinement-compatible token-axis normalization when content is copied. For the MoE auxiliary loss
$E\sum_e f_e\bar P_e$, replace token fractions and average routing probabilities by
$f_e=\sum_i m_i\mathbf1[r_i=e]/M$ and $\bar P_e=\sum_i m_iP_{ie}/M$, with $M=\sum_i m_i$.
Both quantities are unchanged when copies retain the parent's routing probabilities and assignment.

\paragraph{Budgets.}
Consider six equal-mass tokens routed as $(0,0,0,1,1,1)$ to two experts with count capacity three.
Splitting the first token into two copies creates four requests for expert zero without increasing
its rounded capacity; an untouched token can be rejected. A budget $c\sum_i m_i/E$, charged in
units of each token's mass, preserves admission in this constructed case. Total mass is invariant,
but that fact alone does not make every discrete routing algorithm covariant. At a capacity
boundary, independent all-or-nothing decisions for children can differ from the parent's decision.
A general construction must additionally preserve groups of equivalent copies or define a
consistent allocation of fractional measure.

\paragraph{Visibility and score transforms.}
A fixed number of neighboring tokens does not define a fixed physical window after refinement.
A physical window resolves that mismatch, subject to how boundary cells are integrated.
Causality likewise requires comparing the same visible region, rather than copying a parent answer
to a child with a different prefix. Temperature or softcapping can be applied to the score before
measure weighting. None of these mask or budget choices is addressed by a log-measure bias alone.

\subsection{Floating-point and training reproducibility}\label{app:repro}
All bitwise statements in the experiments concern the recorded float32 inference controls and their
specific execution paths. Algebraic split invariance allows rounding differences when attention
length or kernel dispatch changes. In the single-clip split probe, a constant-bias control and a
zero-value length control directly measure the numerical scale of this effect.

A separate zero-mask control evaluates the same eight clips through the default and dense-mask
attention paths. Float32 logits are bitwise identical on these controls; bfloat16 logits differ on
all eight clips, with relative logit deviations up to approximately $1.1\times10^{-2}$ under the
recorded normalization. This supports an execution-path explanation for sensitivity to an
algebraically zero bias.

The synthetic merge and gain datasets used here record the deterministic math attention backend.
Other trained comparisons are summarized across paired seeds. Pairing reduces variation from
initialization and data order; it neither proves bitwise reproducibility nor supplies extra
independent training runs. Inference-time identities and trained accuracy comparisons are therefore
reported with different uncertainty statements.

\section{Statistical methods}\label{app:statistics}
\subsection{Estimands and units}\label{app:estimands}
For trained synthetic models, a paired observation is the difference between arms with the same
initialization and data seed. When the estimand averages multiple rates, we average within seed
first. For a sample of $n$ paired differences $d_i$, the reported interval is
$\bar d\pm t_{n-1,0.975}s_d/\sqrt n$. Thus repeated rates increase the detail of a response curve,
but do not increase the number of independent training runs. Error bars labeled SE use
$s_d/\sqrt n$; all other uncertainty bars are explicitly labeled 95\% confidence intervals.

For inference on benchmark questions, accuracy compares paired correctness indicators.
The exact two-sided McNemar test conditions on the number of discordant pairs and tests an equal
probability of a gain or loss. An answer changing between two incorrect options does not contribute
to this test. The approximate paired intervals in the tables describe variation across questions;
they do not incorporate training variation or guarantee coverage under within-video dependence.
Where one source video contributes several questions, the deployment contrasts of
Appendix~\ref{app:prospective} therefore cluster at the video level rather than relying on these
question-level tests.
A difference with zero observed variance is an identity in the sample, and a nonsignificant
test establishes no equivalence.

The frame-budget KL endpoint averages the five nonreference budgets within each question.
The main-text interval uses a cluster-robust standard error over 534 source videos, with
533 degrees of freedom. For cluster $g$ with $n_g$ questions and sum $S_g$ of paired differences,
its variance estimate is $\frac{G}{G-1}\sum_g(S_g-n_g\bar d)^2/n^2$, where $G=534$ and $n=540$.
Video identifiers shared by STAR and Charades-STA are grouped together; equal basenames from
different source datasets remain distinct. A percentile bootstrap supplements this interval:
sample 534 video clusters with replacement, retain all questions in each sampled cluster, and
recompute the question-weighted mean. With 20,000 resamples and random seed 9102026, the interval
is $[0.0080,0.0260]$ nats, consistent with the cluster-robust interval
$[0.0077,0.0259]$. The corresponding bootstrap accuracy interval is
$[-1.17,+0.56]$\,pp. WorldSense accuracy intervals in
Table~\ref{tab:worldsense} use the same bootstrap procedure over 536 source-video clusters,
with benchmark identifiers joined to the dataset metadata. Other benchmark intervals are
conditional on the evaluated question set and its sampling assumptions.

\subsection{Multiplicity and exploratory results}\label{app:multiplicity}
The controlled studies use a retrospective family of 54 contrasts, with Bonferroni-adjusted
$p$-values $\min(54p,1)$ and nominal threshold $0.05/54 = 9.26\times10^{-4}$.
This family provides supportive evidence. It contains every inferential contrast of the
controlled studies in the analysis registry: the measure-exponent identifiability battery ($9$ contrasts);
the cross-rate factorial scope, mechanism, resolution-floor, and seed-level tests ($17$); the
token-merging, extent-decomposition, and transfer comparisons ($12$); the learned-bias and
loss-weighting controls ($4$); the WorldSense calibration doses ($3$); the released-model
cross-panel equivalences ($6$); the natural frame-budget drift reduction ($1$); and the
lattice-dose correlation ($2$). Each benchmark's deployment evaluation uses a separate
four-contrast family and primary budget specified in the report generator before data collection
(Appendix~\ref{app:deployment_inference}). Cellwise intervals, ablations, window selection,
and extended benchmark probes are exploratory unless explicitly stated otherwise.

For the factorial, the 17 ascents describe adjacent differences in the \emph{mean} response.
The inferential statistic instead averages each seed's number of strict ascents over its
20 adjacent comparisons. We enumerate all $2^{10}$ sign changes of the seedwise gain matrices,
recompute that statistic, and use the upper-tail fraction. Ties remain ties under sign reversal.
The result is $p = 2/1024 \approx 0.00195$, which exceeds the conservative threshold.
The sign-flip test assumes exchangeability of arm labels within seed under its null.
It does not treat adjacent cells as independent trials.

The six-seed attention-slope test gives nominal $p = 7.64\times10^{-5}$ against zero
(adjusted $p = 0.0041$). This supports a nonzero response, not equivalence of the slope to one.
The frame-budget KL difference gives nominal $p = 2.90\times10^{-4}$ paired over questions
($2.96\times10^{-4}$ under the video-clustered primary test; adjusted $p = 0.016$ either way).
The eight-seed full-model versus shared-clock transfer comparison gives nominal $p = 0.020$,
above the family threshold. Each result addresses its specified endpoint and comparison.

\subsection{Independent factorial replication}\label{app:factorial_replication}
The factorial panel of Appendix~\ref{app:synthetic} was evaluated with thirty independent
training seeds (seeds 10--39, ninety arm--seed fits) under the same training recipe as the
ten-seed panel. Its analysis family is separate from the 54 controlled-study contrasts and
the four deployment contrasts; their multiplicity budgets do not cover this replication.
Nineteen of the twenty adjacent steps of the mean gain increase with the rate ratio, against
seventeen in the ten-seed panel, and the exact sign-flip test on seedwise ascents gives
$p = 9.3\times10^{-10}$, the attainable minimum at thirty seeds (ten seeds: $p = 2/1024$).
The pooled gain at equal rates is $+0.52$\,pp with 95\% CI $[-0.65,+1.68]$,
an unresolved difference (ten seeds: $+0.09$\,pp). At ratio 32 the gain is
$+16.8$\,pp with 95\% CI $[+13.5,+20.2]$, which contains the ten-seed point
estimate of $+14.2$. Regressing the cellwise gain jointly on the theoretical dose
$\mathrm{std}(\log m)$ and on $\log(1/r_{\mathrm{slow}})$ gives a dose coefficient of
$+16.7 \pm 2.1$\,pp and a coarseness coefficient of $+0.4 \pm 0.6$\,pp (mean $\pm$ SE;
ten seeds: $+13.5 \pm 5.1$ and $+1.3 \pm 1.0$). Gain increases with dose in this regression;
the smaller coarseness estimate does not establish a zero effect. Cellwise mean gains of the
two panels correlate at $r = 0.981$.

A secondary comparison against learned per-modality bias in the thirty-seed panel also gives
nineteen of twenty ascents and exact sign-flip $p = 1.9\times10^{-9}$. Its equal-rate gain is
$+1.15$\,pp with 95\% CI $[+0.28,+2.03]$ (nominal $p = 0.012$), and its coarseness coefficient
is $+1.44 \pm 0.56$\,pp (SE; nominal $p = 0.016$). Thus the learned-bias comparison includes
both a gain at equal rates and an association with absolute resolution; rate imbalance alone
does not account for it.

\subsection{Observations and exclusions}\label{app:exclusions}
Numerical summaries are rebuilt from saved observations, with input paths and SHA-256 hashes
recorded in the analysis manifest. Paired records are joined by seed or question identity;
overlapping shards are counted once. The figure data record the analyzed sample sizes.
In the large duplication evaluation, 3,586 of 3,600 eligible
questions have usable predictions from 18 tasks; 14 questions with unavailable media are excluded before paired
comparisons. The two-region, frame-budget, WorldSense, and multiscale summaries use complete
records at their stated sample sizes.

The deterministic quadrature illustration is regenerated from its analytic functions, with
adaptive integration for reference values. Floating-point primitive and split controls are
numerical checks of algebraic statements, so they receive tolerances and maximum deviations
rather than confidence intervals. The accompanying data manifest and scripts distinguish
these constructions from model observations.

\section{CPU experiments on natural partitions}\label{app:cpu}
\subsection{A fixed nonlinear attention functional}\label{app:cpu_generator}
The domain is the unit circle with Lebesgue measure. We generate 512 fields with NumPy seed
2026091017. On a 4,096-cell uniform grid with midpoint coordinates $x$, the signal is
\[
 f(x)=b+\sum_{e=1}^8a_e\exp[-d_{\rm circ}(x,c_e)^2/(2\sigma_e^2)],
\]
where $b\sim\mathcal N(0,0.35^2)$, $a_e\sim\mathcal N(0,1)$,
$c_e\sim U[0,1)$, and $\sigma_e\sim U[0.015,0.12]$, independently.
The nuisance channel is $u(x)=\sum_{e=1}^8z_e\sin(2\pi\nu_ex+\theta_e)$ with
$z_e\sim\mathcal N(0,0.15^2)$, integer $\nu_e$ uniform on $\{1,\ldots,16\}$, and
$\theta_e\sim U[0,2\pi)$. We set
\[
 s(f,u)=1.2\tanh f+0.2u,\qquad v(f,u)=\tanh(f-0.4u).
\]
All partitions average $f$ and $u$ over contiguous groups of the same fine cells before applying
these nonlinear maps. Hence changing a cell changes both its score and value. The reference
is the measure-weighted attention of the fine grid itself; no claim of exact continuous integration
is needed for this comparison.

For each field and budget $N\in\{16,32,64,128\}$, we form three complete partitions.
Uniform boundaries divide the grid equally. Activity-driven boundaries are equal quantiles of
$1+8|f|$. Variation-driven boundaries are equal quantiles of $1+8d/\bar d$, where
$d_j=|f_{j+1}-f_j|$ with circular indexing. We locate quantiles by cumulative sums, constrain each
cell to contain at least one fine cell, and retain the exact resulting boundaries. Thus all arms
at a given budget share token contents and full domain coverage. Only the attention weights differ.
For one modality, inverse count multiplies every weight by the same scalar and agrees with standard
attention exactly. Cell mass is the number of fine cells divided by 4,096.

\subsection{Outcomes and coupled verification}\label{app:cpu_results}
Every field contributes twelve observations, which are retained with its identifier and partition
boundaries. Figure~\ref{fig:stability} reports means and Student intervals across the 512 fields
separately at each configuration; it does not pool repeated budgets as independent draws.
The score/value map is fixed; this experiment makes no trained-model or task-accuracy claim.
A fixed zero-threshold decision can also be compared with the dense reference; the saved records
include all decisions, but operator error is the reported endpoint.

To evaluate Theorem~\ref{thm:stability}, lift the coarse score and value back to every fine cell.
The paired value error is their maximum absolute difference from the fine-grid values, and $\Delta$
is the range of the paired score differences. The reference-value range gives $D_v$.
All 6,144 comparisons satisfy the bound. The bound is worst-case: the largest ratio of observed
error to bound is approximately 0.10. Tightness is attained by the two-atom construction of
Appendix~\ref{app:stability}, not on these signals.

At 32 cells, uniform standard and measure weighting both have mean absolute error 0.0021.
Activity-driven errors are 0.1524 and 0.0076, respectively; variation-driven errors are 0.1005 and
0.0020. At 128 cells the adaptive standard errors are 0.1531 and 0.1014, while measure errors are
0.0008 and 0.0005. The discrete fine reference is fixed throughout. All fields, including ones
near the optional decision boundary, remain in the sample.

\subsection{Mathematical numerical checks}\label{app:cpu_checks}
A separate CPU verifier checks arbitrary sparse mass couplings, unequal exact and approximate
splits, sharp two-atom cases, support-Wasserstein perturbations, six-layer propagation, and the
partial-cell visibility and overlapping-cover constructions. It includes negative controls for
count weighting, causal-prefix mismatch, and nonlinear averaging of overlapping features.
For centered asymmetric supports, the general characteristic-function formula agrees with direct
summation to $8.4\times10^{-17}$, while dropping the conjugate gives error approximately 0.155.
These constructed checks supplement the proofs; they are not samples of deployed-model performance.

\section{Controlled temporal experiments}\label{app:synthetic}
\subsection{Signal generation and optimization}\label{app:synthetic_setup}
The timeline is a circle of duration $D=4$\,s. Draw a gap $g$ uniformly from $[0.70,2.00]$\,s,
a center $c$ uniformly on the circle, and an independent sign $\epsilon\in\{-1,1\}$.
Place one event in each modality at $(c-\epsilon g/2)\bmod D$ and
$(c+\epsilon g/2)\bmod D$. The target is
$\lfloor5(g-0.70)/1.30\rfloor$, clipped to classes $0,\ldots,4$.
Each event is a wrapped Gaussian bump with standard deviation 0.10\,s. Token features contain
the cell-average signal and modality indicators; center, extent, and index coordinates are
provided to the corresponding positional components. There are no distractor events.
Wrapping and uniform placement make marginal event position independent of the label.

The encoder uses four layers, six attention heads, width 192, and a five-class readout
(1,773,125 parameters in the basic configuration). Training uses batch size 32 for 1,500 steps,
AdamW with learning rate $10^{-3}$, weight decay 0.01, $\beta=(0.9,0.95)$, a one-cycle schedule,
and gradient-norm clipping at one. The continuum rotary wavelength ladder spans 0.02--40\,s.
The standard positional baseline uses separate modality indices. Other arms introduce a learned
modality bias, physical coordinate features, a shared quantized clock, continuum RoPE, or
measure-weighted attention. The same seed identifies initialization and training-example streams
within each paired experiment.

\paragraph{Label recoverability.}
A template estimator searches event positions on a 5\,ms grid, using wrapped, box-averaged
Gaussian templates and profiling out nonnegative amplitude. On 1,500 examples at each of the
eight transfer rates, it achieves approximately 99.1\% accuracy throughout. This provides an
achievable reference under the specified generator, including cells wider than the event bump;
it is not a Bayes bound.
The remaining errors occur near quantization boundaries. Non-temporal summary-feature probes
score 17.7--19.7\% against a five-class chance level of 20\% in the recorded shortcut checks.

\subsection{Cross-rate transfer}\label{app:transfer}
Training samples the audio/video rate pairs $(8,4)$, $(8,8)$, and $(4,4)$ tokens/s.
Evaluation adds $(16,8)$, $(25,2)$, $(32,16)$, $(3,2)$, and $(50,25)$.
Table~\ref{tab:transfer} and Figure~\ref{fig:transfer} average those five held-out rates within seed.
The component comparison uses five seeds and the stronger-baseline comparison eight, as indicated
in Table~\ref{tab:transfer}; inferential comparisons use the paired subsets.

\begin{figure}[!htbp]
\centering\includegraphics[width=\textwidth]{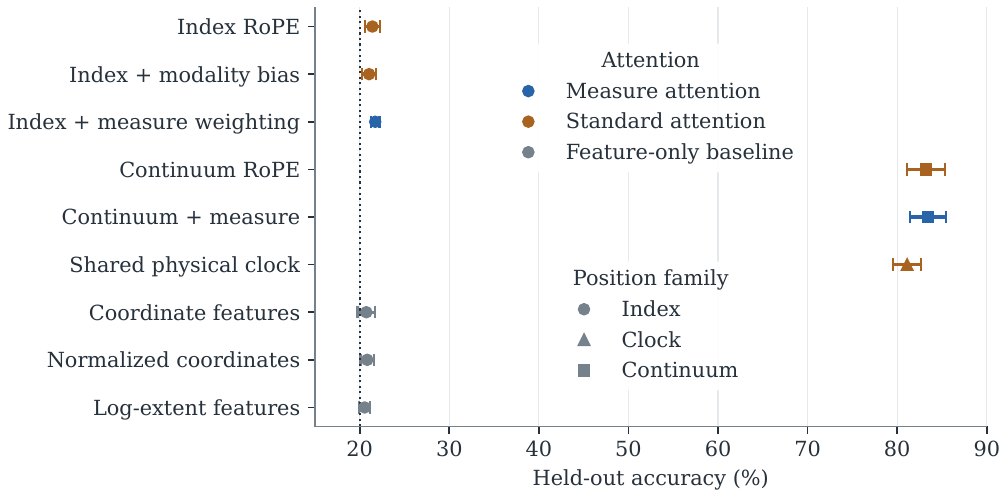}
\caption{\textbf{Physical position drives cross-rate transfer.} Points and error bars show seed
means and 95\% intervals for held-out accuracy. Blue denotes measure weighting, gold standard
attention, and grey feature baselines; marker shape encodes the positional family. The dotted
line marks chance. Sample sizes and in-distribution scores appear in Table~\ref{tab:transfer};
comparisons between independently listed summaries are descriptive.}
\label{fig:transfer}
\end{figure}

\begin{table}[H]
\centering\small
\caption{\textbf{Physical clocks carry cross-rate accuracy.} Training rates are $8\times4$, $8\times8$, and $4\times4$ tokens/s; the held-out column averages five other rates within each seed before computing its 95\% interval. The eight-seed paired clock comparison is reported separately.}
\label{tab:transfer}
\begin{tabular}{@{}lrrrr@{}}
\toprule
Arm & Seeds & Trained (\%) & Held out (\%) & 95\% CI \\
\midrule
Index RoPE & 5 & 92.5 & 21.4 & $[20.6, 22.2]$ \\
Index + modality bias & 5 & 93.1 & 21.1 & $[20.2, 21.9]$ \\
Index + measure weighting & 5 & 91.8 & 21.7 & $[21.2, 22.2]$ \\
Continuum RoPE & 5 & 97.2 & 83.2 & $[81.1, 85.3]$ \\
Continuum + measure & 5 & 97.1 & 83.4 & $[81.4, 85.5]$ \\
Shared physical clock & 8 & 97.4 & 81.1 & $[79.6, 82.7]$ \\
Coordinate features & 8 & 81.0 & 20.7 & $[19.8, 21.7]$ \\
Normalized coordinates & 8 & 89.3 & 20.8 & $[20.1, 21.6]$ \\
Log-extent features & 8 & 88.0 & 20.6 & $[19.9, 21.2]$ \\
\bottomrule
\end{tabular}
\end{table}

The matched eight-seed comparison gives a full-model minus shared-clock gain of 2.94\,pp, with
95\% CI $[0.61,5.26]$ and nominal $p = 0.020$. It does not meet the conservative
multiplicity threshold. A clock-resolution sweep also shows that finer ticks need not improve
task accuracy monotonically. For example, the recorded mean held-out scores at a 0.04\,s tick
are 81.4\% with rounding and 84.4\% with flooring; the continuum arm scores 83.8\% in that
experiment. At a 0.5\,s tick the corresponding clock scores are 68.8\% and 83.8\%.
These responses are specific to this task and positional convention.

\subsection{Factorial uncertainty}\label{app:factorial}
The $4\times6$ factorial crosses slow rates $r_s\in\{2,4,8,16\}$ with ratios
$q\in\{1,2,4,8,16,32\}$, using $(r_a,r_v)=(qr_s,r_s)$.
Both compared arms use continuum RoPE; only measure weighting differs. Ten paired seeds provide
cellwise estimates in Table~\ref{tab:factorial_intervals}. A separate six-seed log-odds
measurement supplies the slope test of Appendix~\ref{app:statistics}.
Training covers ratios one and two. The larger-ratio response is thus also a transfer result.
Increasing $q$ raises sequence length as well as imbalance; the factorial holds the slow rate
fixed within each row but cannot make length independent of the two rate variables.

\begin{table}[!htbp]
\centering\footnotesize
\caption{\textbf{Cellwise gains across rate imbalance.} Cellwise accuracy gain from measure weighting, in pp. Each interval uses ten paired seeds and is pointwise, not simultaneous across cells.}
\label{tab:factorial_intervals}
\begin{tabular}{@{}rrrrrrrr@{}}
\toprule
Slow rate & Ratio & Gain & 95\% CI & Slow rate & Ratio & Gain & 95\% CI \\
\midrule
2 & 1 & -1.25 & $[-8.19,+5.69]$ & 8 & 1 & +0.20 & $[-0.05,+0.44]$ \\
2 & 2 & +0.43 & $[-3.36,+4.22]$ & 8 & 2 & +0.14 & $[-0.34,+0.63]$ \\
2 & 4 & +2.40 & $[-0.62,+5.42]$ & 8 & 4 & +1.07 & $[-1.21,+3.34]$ \\
2 & 8 & +5.22 & $[+2.83,+7.61]$ & 8 & 8 & +3.75 & $[-2.74,+10.24]$ \\
2 & 16 & +10.47 & $[+5.53,+15.40]$ & 8 & 16 & +7.68 & $[-1.65,+17.02]$ \\
2 & 32 & +15.94 & $[+9.88,+22.00]$ & 8 & 32 & +12.45 & $[+1.62,+23.28]$ \\
4 & 1 & +1.51 & $[-1.54,+4.56]$ & 16 & 1 & -0.09 & $[-1.01,+0.83]$ \\
4 & 2 & +0.35 & $[-1.05,+1.75]$ & 16 & 2 & -0.76 & $[-2.96,+1.45]$ \\
4 & 4 & +1.42 & $[+0.36,+2.48]$ & 16 & 4 & +0.44 & $[-2.05,+2.94]$ \\
4 & 8 & +5.00 & $[+0.12,+9.88]$ & 16 & 8 & +3.27 & $[-1.56,+8.09]$ \\
4 & 16 & +11.56 & $[+1.87,+21.25]$ & 16 & 16 & +6.48 & $[-1.21,+14.18]$ \\
4 & 32 & +18.84 & $[+7.73,+29.95]$ & 16 & 32 & +9.69 & $[-0.02,+19.39]$ \\
\bottomrule
\end{tabular}
\end{table}

\subsection{Token merging and support extent}\label{app:merge}
At inference, merge adjacent same-modality cells at factors $1,2,4,8$. The merged representation
uses measure-weighted feature averaging, its center is the measure-weighted center, and its
extent and measure cover the merged interval. The shared-clock and continuum arms use the same
physical centers. Each support profile is fitted separately with standard attention and matched
initialization and data seeds; merging is then applied at inference. Figure~\ref{fig:merge}
separates position from attention weighting and then compares the extent profiles across eight seeds.

\begin{figure}[!htbp]
\centering\includegraphics[width=\textwidth]{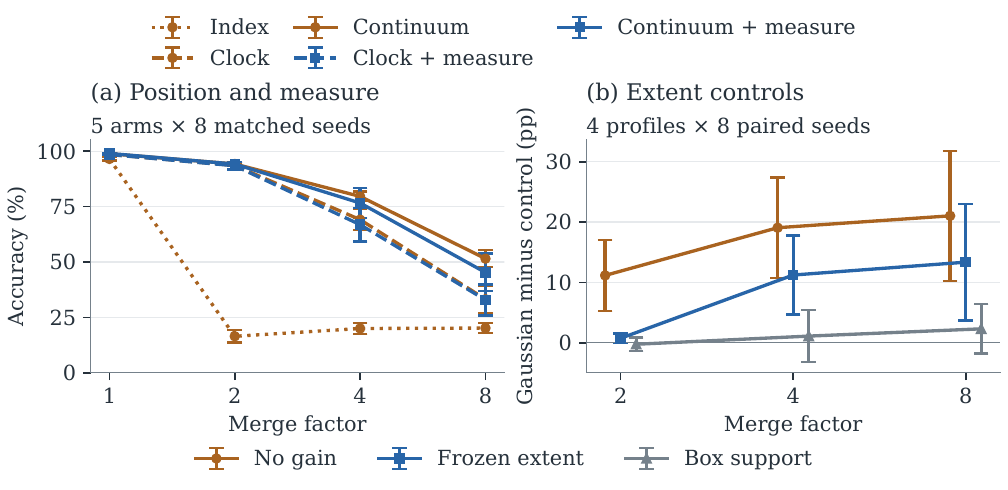}
\caption{\textbf{Support-aware positional models retain more accuracy under merging.} (a) Accuracy at each merge factor
for five positional and attention configurations, matched across eight seeds. (b) Paired gain of
a width-tracking Gaussian over no attenuation, a Gaussian retaining the original extent, and a
width-tracking box profile. These four profiles are separately trained with standard attention
using eight matched initialization and data seeds; merging is applied at inference. Error bars are
95\% intervals. In (a), blue denotes measure weighting and gold standard attention; line style
encodes the positional family.}
\label{fig:merge}
\end{figure}

Models trained without extent attenuation have lower accuracy by 11.1, 19.1, and 21.0\,pp at factors
2, 4, and 8. Keeping the original extent reduces accuracy by 0.7, 11.2, and 13.3\,pp.
Gaussian-minus-box differences are $-0.3$, $+1.1$, and $+2.3$\,pp, with nominal paired
$p = 0.61$, $0.57$, and $0.22$. The evidence favors tracking represented extent on this task;
it does not identify a uniquely superior support distribution. In the uniform same-modality merge
comparison, counts and measures remain closely linked, so no benefit of measure weighting is
expected or resolved.

\section{Released-model interventions}\label{app:released}
\subsection{Rate-dependent attention allocation}\label{app:rate_panel}
The Qwen2.5-Omni-7B panel evaluates 16 clips at video rates $1,2,4,8,16,32$\,fps with audio
tokenization fixed. The model weights are unchanged. Attention odds and shares are recorded for
the same probe under standard attention and a native-calibrated log-measure bias.
Video share is $V/(V+A)$, conditional on attention to audio or video. Within each layer,
$V$ and $A$ are averaged over heads and the final eight prompt queries; shares and odds are then
averaged over layers. Each clip contributes one slope from regressing log odds on log video rate. Mean slopes are
$0.4646$ (95\% CI $[0.4391,0.4902]$) and $0.0536$ ($[0.0332,0.0740]$), respectively.
The relative slope reduction is approximately 88\%.
The positive residual reflects sensitivity beyond the counting term; the correction does not
make resampled video representations identical. Figure~\ref{fig:rate-allocation} reports clip-level
uncertainty for the attention shares rather than treating heads as independent samples.

\begin{figure}[!htbp]
\centering\includegraphics[width=\textwidth]{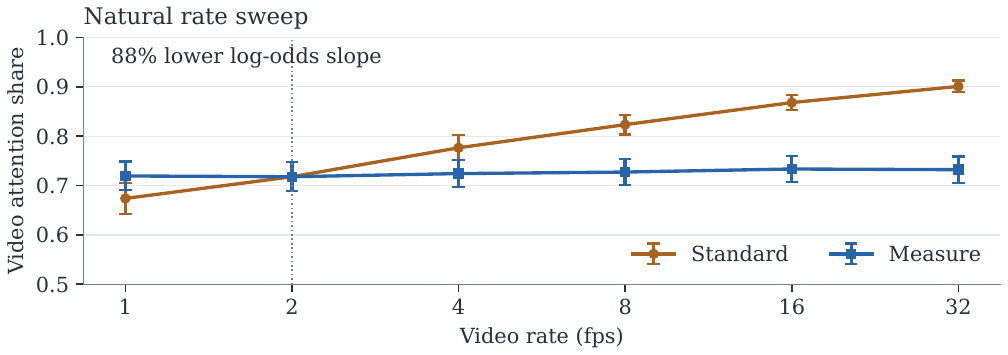}
\caption{\textbf{Measure weighting reduces rate-dependent attention allocation.}
Qwen2.5-Omni-7B's video share of audio/video attention, averaged over 16 clips with fixed audio tokenization.
Error bars show 95\% $t$ intervals across clips. Orange circles denote standard attention;
blue squares denote native-calibrated measure weighting. The dotted line marks the 2\,fps
calibration rate.}
\label{fig:rate-allocation}
\end{figure}

\subsection{Duplication with controlled visibility}\label{app:duplication}
For a fixed query, duplicate a selected key/value representation $k$ times, keep its positional
representation, and assign each copy $1/k$ of the parent's reference mass. In a multilayer model,
copies must also have the parent's visible context and appear identically to later queries.
The matched-visibility intervention enforces these conditions. Its standard-attention control
duplicates the representation without reducing its weight.

The float32 full-modality probe uses $k=2,3,4$ on one clip. Standard attention multiplies the
measured audio/video competition by the predicted factor; measure-weighted logit deviations remain
within 0.91--1.19 times the numerical control floor. A single clip supplies an implementation
check, not a precision estimate.

The larger MVBench experiment duplicates a contiguous half of the visual representations while
preserving the other half. Of 3,600 eligible questions, 3,586 have usable paired predictions;
14 questions have unavailable media.
At $k=2$, standard attention changes 150 answers: 63 correct-to-incorrect, 55 incorrect-to-correct,
and 32 between incorrect options. At $k=3$, the corresponding counts are 255, 111, 90, and 54.
Measure weighting changes no answers at either factor under matched visibility. The exact
McNemar $p$-values for standard attention's accuracy change are 0.52 and 0.16.
Measure-weighted logits need not be bitwise identical even when answer identity is preserved.

A strictly causal control uses 720 questions. Later copies have longer prefixes than earlier
copies, so the duplicated group no longer has one shared visible context. The measure-weighted
arm changes four answers (0.56\%). This illustrates the theorem's visibility condition; a causal
refinement must specify the physical horizon at which an output is evaluated.

\paragraph{Audio-Flamingo-3.}
The second architectural probe operates on the released language backbone using 64 audio slots,
24 text slots, and 40 independent representation draws in float32. It duplicates the audio slots
at factors 2, 3, and 4. Thus it evaluates the fusion computation on controlled representations,
not the end-to-end accuracy of the audio encoder. Under matched visibility it exhibits the same
odds scaling and its removal by measure weighting. Strictly causal copies again change the
operator; their remaining deviations reflect the changed visible context, consistent with the
visibility condition.

\subsection{Unequal sampling and count controls}\label{app:unequal}
The two-region experiment uses 40 questions from each of 18 MVBench tasks. Each clip supplies
16 frames, with $10/6$, $12/4$, or $14/2$ frames assigned to equal-duration halves.
Spatial patching is fixed. The first-layer dense/sparse odds ratio of measure weighting (MC) to
standard attention (std) satisfies
\[
\frac{O_\mc}{O_\std}=\frac{m_D}{m_S}
\in\left\{\frac6{10},\frac4{12},\frac2{14}\right\}
\]
because scores are shared at the intervention. Of 60,480 clip--head--allocation observations,
60,479 have relative error below 1\%; the remaining case has extreme odds and is sensitive to
floating-point division. A single count for the whole video modality cannot distinguish these
allocations. A count rule given the two equal-span groups can reproduce their masses; the result
therefore separates measure weighting from a \emph{per-modality} count rule, not every possible
geometrically informed estimator of density.

More generally, if measures vary within two probe groups, the odds ratio becomes
\[
\frac{O_\mc}{O_\std}
=\frac{\sum_{j\in D}p_{j\mid D}m_j}{\sum_{j\in S}p_{j\mid S}m_j},\qquad
p_{j\mid G}=\frac{e^{s_j}}{\sum_{k\in G}e^{s_k}}.
\]
It is bounded by $\min_Dm/\max_Sm$ and $\max_Dm/\min_Sm$, but generally depends on scores.
In a geometric-width schedule on 720 clips, all 40,320 head--allocation observations lie
inside these bounds: approximately $[0.250,0.820]$ and $[0.125,0.743]$ for the two tested schedules.
An inverse local neighbor count can approximate density under suitable sampling assumptions;
it is not exactly additive under arbitrary representation duplication because the neighboring
counts of unsplit cells can also change.

\subsection{Natural frame-budget changes}\label{app:frame_budgets}
The 540-question budget experiment uses $4,8,16,24,32,48$ frames from the same clip and a
16-frame reference answer distribution. Answer logits are restricted to the question's options
and normalized before computing $\operatorname{KL}(p_{16}^{\rm std}\|p_b^a)$ for arm $a$ and
budget $b$. Both arms share the native standard-attention reference; the endpoint averages the
five nonreference budgets within each question. The physical-cell measure bias and the reference-calibrated
count bias are equal at all 3,240 question--budget evaluations.
Figure~\ref{fig:budgets} shows both paired differences and the distribution at the largest budget.

The mean reduction is 0.0168 nats. The mean accuracy change over the five nonreference budgets is
$-0.30$\,pp (95\% video-clustered CI $[-1.16,+0.56]$; nominal $p = 0.50$).
At least one budget changes the answer on 195 questions with standard attention and 186 with
measure weighting. There are 23 measure-only and 32 standard-only changes, giving exact
McNemar $p = 0.28$. The distributional reduction is resolved; accuracy and answer-consistency
differences remain unresolved.

\subsection{WorldSense calibration}\label{app:ws_calibration}
Each arm evaluates the same 600 questions at $2,0.5,0.25$\,fps. The native-calibrated measure
arm uses video bias $\log(2/\mathrm{fps})$ based on the realized frame rate, with the audio reference fixed. The calibrated count
arm uses the matching reference counts.
The uncalibrated count arm instead changes the relative modality prior even at the reference
rate. A constant-bias control, held at its reference-rate value across the sweep, changes
42, 28, and 34 answers. Table~\ref{tab:worldsense} reports answer changes and paired accuracy
differences separately. At the native rate, an algebraically zero bias may
still select a different low-precision execution path (Appendix~\ref{app:repro}).

\begin{table}[!htbp]
\centering\footnotesize
\caption{\textbf{Calibrated attention under WorldSense rate changes.} Each rate evaluates the same 600 questions. The change columns count answers differing from standard attention for measure weighting (MC) and for the uncalibrated and calibrated count arms. Calibrated count and measure weighting give identical answers. Accuracy differences (measure minus standard) are in pp with 95\% bootstrap intervals grouped by source video.}
\label{tab:worldsense}
\begin{tabular}{@{}rrrrrr@{}}
\toprule
fps & Standard (\%) & MC & Count (uncal.) & Count (cal.) & $\Delta$ accuracy [95\% CI] \\
\midrule
2 & 34.7 & 9 & 65 & 9 & $+0.00\;[-0.66,+0.67]$ \\
0.5 & 36.2 & 48 & 18 & 48 & $-0.67\;[-2.51,+1.17]$ \\
0.25 & 33.3 & 70 & 22 & 70 & $+0.50\;[-1.81,+2.84]$ \\
\bottomrule
\end{tabular}
\end{table}

\subsection{Generality across model scale}\label{app:qwen3b}
The probes above were repeated on a second checkpoint, Qwen2.5-Omni-3B, using the identical
harness, question sets, and aggregation as the 7B study; the checkpoint is the only change.
Table~\ref{tab:qwen3b} places the two checkpoints side by side. On the 3{,}586-question
duplication panel, standard attention changes 140 answers at split factor two and 236 at factor
three, while the measure-weighted arm changes none, with a maximal answer-distribution
perturbation of $1.3\times10^{-9}$ nats against the un-duplicated reference; the $k=1$ null
control is exactly zero in both arms. The strictly causal control leaves the same small residual
as on the 7B: four and five measure-arm flips of 720 at factors two and three. On the
540-question frame-budget ladder, the measure term reduces
$\operatorname{KL}(p_{16}^{\rm std}\|p_b^a)$, averaged over the five nonreference budgets,
by 0.0154 nats per question (95\% interval $[0.0082,0.0226]$,
cluster-robust over 534 source videos), with paired accuracy and any-flip differences of
$-0.19$\,pp ($p = 0.72$) and $-2.22$\,pp ($p = 0.11$). As on the 7B, the calibrated count arm
coincides with the measure arm numerically under these uniform rate changes. These measurements
describe the 3B checkpoint only; as single-run probes they are descriptive and enter no
confirmatory family.

\begin{table}[!htbp]
\centering\small
\caption{\textbf{Duplication and resampling across model scales.} Probes on
Qwen2.5-Omni-7B and Qwen2.5-Omni-3B under the identical harness. Flips count changed answers
against each model's own un-duplicated reference; std is standard attention, mc the
measure-weighted arm. The KL row reports standard minus measure-weighted
$\operatorname{KL}(p_{16}^{\rm std}\|p_b^a)$, averaged over the five nonreference
budgets, with 95\% video-clustered intervals. Both arms share the standard-attention reference.}
\label{tab:qwen3b}
\begin{tabular}{@{}lrr@{}}
\toprule
Probe & 7B & 3B \\
\midrule
Duplication std flips at factor two (of 3{,}586) & 150 & 140 \\
Duplication std flips at factor three & 255 & 236 \\
Duplication mc flips at factors two and three & 0, 0 & 0, 0 \\
Causal-copy mc flips at factors two and three (of 720) & 4, 4 & 4, 5 \\
Budget KL reduction, nats [95\% interval] & 0.0168 & 0.0154 \\
 & [0.0077, 0.0259] & [0.0082, 0.0226] \\
\bottomrule
\end{tabular}
\end{table}

\section{Extended evaluations}\label{app:extended}
\subsection{Content-adaptive video tokenization}\label{app:adaptive}
Three tokenizers select cells from the same video interval at a fixed total budget.
The motion tokenizer partitions cumulative inter-frame difference; the entropy tokenizer uses
per-frame spatial entropy; the keyframe tokenizer places boundaries at large motion peaks.
Each token carries its realized temporal extent. The median within-clip extent ratio is about
three, with a 90th percentile of 8.5. Covered duration, budget, and mean declared frame rate are
matched across the three tokenizers, while selected frames can differ.

The evaluation contains 432 questions from 17 MVBench tasks. Consistency is the fraction of
questions whose chosen answer agrees under all three tokenizations. Table~\ref{tab:adaptive} compares standard
attention, modality count, measure weighting, and local count.
The local-count window is chosen descriptively from
$\{0.1,0.2,0.35,0.5,0.75,1,1.5,2\}$\,s; its best observed consistency occurs at 0.75\,s, a choice
made on the same data and therefore optimistically biased. The resulting consistency differences
are unresolved.

\begin{table}[!htbp]
\centering\footnotesize
\caption{\textbf{Consistency across content-adaptive tokenizations.} Evaluation on 432 MVBench questions. Consistency is agreement across all three tokenizers; accuracy is averaged within question over tokenizers. The interval and nominal exact McNemar test compare consistency with standard attention. The local-count window is selected descriptively from eight tested windows.}
\label{tab:adaptive}
\begin{tabular}{@{}lrrrr@{}}
\toprule
Arm & Consistency (\%) & Accuracy (\%) & $\Delta$ consistency (pp) & $p$ \\
\midrule
Standard & 86.1 & 67.7 & $+0.00$ & 1.000 \\
Modality count & 87.0 & 68.9 & $+0.93\;[-1.73,+3.58]$ & 0.608 \\
Measure & 87.0 & 68.4 & $+0.93\;[-1.88,+3.73]$ & 0.627 \\
Local count (0.75 s) & 86.6 & 68.1 & $+0.46\;[-1.47,+2.39]$ & 0.815 \\
\bottomrule
\end{tabular}
\end{table}

\subsection{LoRA fine-tuning on MVBench}\label{app:trained}
LoRA rank 16 adapts the query, key, value, and output projections of all 28 text-decoder layers
of Qwen2.5-Omni-7B. Base and measure-weighted arms share training examples and seed within each
comparison. The 120- and 360-question evaluations use 400 training examples and 200 steps;
the larger 2,786-question evaluation uses 600 examples and 300 steps. Training uses AdamW with
learning rate $10^{-4}$ and a one-cycle schedule. Each evaluation compares ratios $q=1$ and $q=8$.
Table~\ref{tab:trained} reports the paired correctness counts and actual evaluation denominators.
Each slice is question-disjoint between training and evaluation, but not source-video-disjoint:
75 of the 589 training videos recur in the 2{,}786-question evaluation (covering 78 of its 600
training questions), and two and one training videos recur in the 120- and 360-question
evaluations, respectively. These exploratory estimates are conditional on the stated splits; the primary deployment
evaluation uses disjoint source videos (Appendix~\ref{app:prospective}).

\begin{table}[!htbp]
\centering\footnotesize
\caption{\textbf{Accuracy after low-rank adaptation.} LoRA evaluations on MVBench. The 120- and 360-question slices use 400 training examples and 200 steps; the 2,786-question slice uses 600 examples and 300 steps. Each is one trained seed per arm. Intervals describe paired question variability conditional on those checkpoints. W/L counts correctness gains/losses for measure weighting; $p$ is nominal exact McNemar.}
\label{tab:trained}
\begin{tabular}{@{}rrrrrrr@{}}
\toprule
$n$ & $q$ & Standard (\%) & MC (\%) & $\Delta$ (pp) [95\% CI] & W/L & $p$ \\
\midrule
120 & 1 & 61.67 & 61.67 & $+0.00$ & 0/0 & 1.000 \\
120 & 8 & 61.67 & 64.17 & $+2.50\;[-2.45,+7.45]$ & 6/3 & 0.508 \\
360 & 1 & 64.44 & 63.61 & $-0.83\;[-1.78,+0.11]$ & 0/3 & 0.250 \\
360 & 8 & 63.89 & 62.78 & $-1.11\;[-3.79,+1.57]$ & 10/14 & 0.541 \\
2786 & 1 & 62.42 & 62.67 & $+0.25\;[-0.11,+0.62]$ & 17/10 & 0.248 \\
2786 & 8 & 62.31 & 63.03 & $+0.72\;[-0.16,+1.59]$ & 87/67 & 0.125 \\
\bottomrule
\end{tabular}
\end{table}

On 2,786 questions, the $q=8$ difference is $+0.72$\,pp, with 95\% CI
$[-0.16,+1.59]$ and nominal exact McNemar $p = 0.125$. The 120-question $q=1$ slice has
identical recorded correctness indicators in the two arms. An inference-time zero-bias identity
does not require separately optimized models to have identical parameters.
With one trained seed per arm, these intervals characterize question variability only.

\subsection{Modality weights in a masked-token objective}\label{app:loss}
A separate task samples a shared piecewise-constant latent signal on a four-second timeline.
Masked codes in either modality can be predicted using the other modality's observations of
the same interval. Video rate is fixed at four tokens/s and audio rate is $q$ times larger,
with $q\in\{1,2,4,8,16\}$. For each objective, nine audio weights
$w\in\{1/16,1/8,1/4,1/2,1,2,4,8,16\}$ and five seeds give 225 runs.
Training uses 1,500 steps, batch size 32, and learning rate $10^{-3}$.
Both objectives are evaluated with the same measure-normalized held-out metric.

Let $S$ be the selected prediction targets, $n_a,n_v$ their modality counts, and
$\ell_i$ their token losses. With $m_a=m_v/q$,
\begin{align*}
L_{\mathrm{tok}}(w)&=\frac{w\sum_{i\in S_a}\ell_i+\sum_{i\in S_v}\ell_i}{n_a+n_v},\\
L_{\mathrm{meas}}(w)&=\frac{(w/q)\sum_{i\in S_a}\ell_i+\sum_{i\in S_v}\ell_i}{n_a/q+n_v}.
\end{align*}
For fixed logits and targets,
$L_{\mathrm{meas}}(w)=\frac{n_a+n_v}{n_a/q+n_v}L_{\mathrm{tok}}(w/q)$.
The factor is independent of $w$ but can vary with the selected target counts. This identity
relates objectives at fixed predictions; it does not imply identical optimization trajectories
or an exact relation between the best trained models on a finite grid.

\begin{figure}[H]
\centering\includegraphics[width=\textwidth]{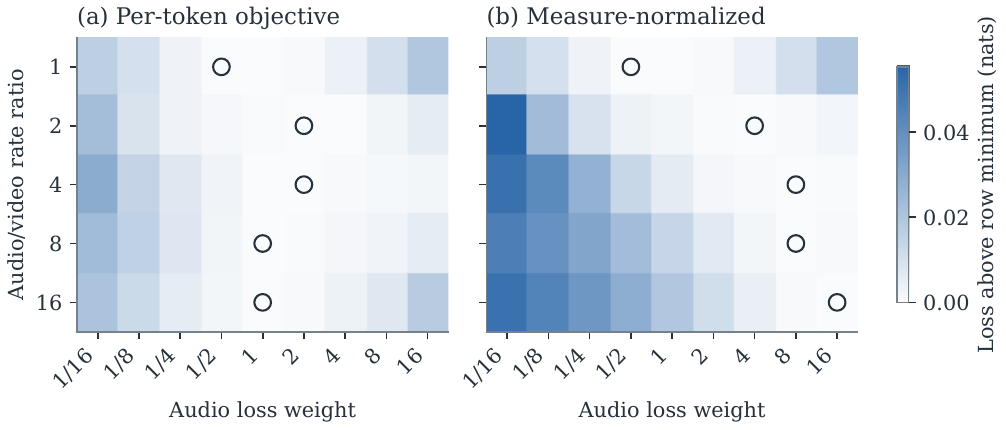}
\caption{\textbf{Loss weighting remains task dependent.} Each cell shows mean held-out loss
above the best tested weight at the same rate, across five seeds: (a) per-token objective,
(b) measure-normalized objective. Circles mark the grid minima.
The shared color scale covers the full observed range; each objective contains 225 runs.
The measure-objective minimum reaches the search boundary at $q=16$, so no extrapolated optimum
or fitted scaling law is reported.}
\label{fig:loss_weights}
\end{figure}

Observed best per-token weights are $(0.5,2,2,1,1)$ across increasing $q$; the measure-objective
grid gives $(0.5,4,8,8,16)$ (Figure~\ref{fig:loss_weights}). Directly setting $w=1/q$ in the per-token objective produces higher
held-out loss than $w=1$ at $q=4,8,16$ (Table~\ref{tab:loss_compensation}). The corresponding
nominal paired tests remain below 0.011 after Bonferroni adjustment over the four nontrivial
rates. This focused exploratory comparison does not use the main-text 54-contrast threshold.
Measure normalization ensures the stipulated refinement behavior of a fixed loss field;
it does not determine the best task tradeoff after learning.

\begin{table}[H]
\centering\small
\caption{\textbf{Inverse-rate loss weighting increases held-out loss at ratios 4--16.} Held-out loss at $w=1/q$ minus loss at $w=1$ in the per-token objective, paired across five seeds. Positive values favor $w=1$. Intervals are 95\%; tests are nominal and two-sided. The $q=1$ comparison is identical by construction.}
\label{tab:loss_compensation}
\begin{tabular}{@{}rrrr@{}}
\toprule
Rate ratio $q$ & Loss difference & 95\% CI & $p$ \\
\midrule
1 & +0.0000 & $[+0.0000,+0.0000]$ & 1.000 \\
2 & +0.0011 & $[-0.0016,+0.0038]$ & 0.312 \\
4 & +0.0068 & $[+0.0040,+0.0096]$ & 0.003 \\
8 & +0.0154 & $[+0.0114,+0.0193]$ & $4.2\times10^{-4}$ \\
16 & +0.0206 & $[+0.0149,+0.0262]$ & $5.3\times10^{-4}$ \\
\bottomrule
\end{tabular}
\end{table}

\subsection{Cross-rate deployment after training}\label{app:deployment}
Six LoRA runs train Qwen2.5-Omni-7B at 2\,fps: three seeds per attention arm, 600 WorldSense
training examples, 200 steps, rank 16, and learning rate $10^{-4}$.
Evaluation uses 500 held-out questions at $0.25,0.5,1,2,4$\,fps with a fixed audio hop.
The complete grid gives 7,500 paired question--rate--seed comparisons.
Baseline mean accuracy spans 37.27--38.73\% across rates, with its minimum at the training rate:
the deployed baseline is already robust at these rates, so the deployment question here concerns
token-budget flexibility rather than degradation repair.

For each seed, regress the accuracy difference (measure minus standard) on
$|\log(\mathrm{fps}/2)|$. The mean slope is $-0.00725$ per natural-log unit, with 95\% CI
$[-0.01576,+0.00126]$ and nominal $p = 0.067$. The sign convention here treats positive slopes
as increasing relative benefit away from the training rate. The result is unresolved across
three seeds, and natural frame changes do not satisfy the fixed-representation split assumptions.

\subsection{Temporal positional quantization}\label{app:lattice}
The released temporal positional construction quantizes a physical clock. To isolate this
operation, hold the decoded image set fixed and vary only the declared frame rate over
$1,2,3,5,6,10,12,24$\,fps on the same 540 MVBench questions. Compare the shipped float32
identifiers with their unquantized counterpart, using otherwise identical weights and inputs.
At rates $1,2,5,10$, the tested positions lie on the shipped lattice and all answers agree.
At $3,6,12,24$\,fps, respectively 1, 4, 2, and 7 answers change (Figure~\ref{fig:lattice}).

\begin{figure}[!htbp]
\centering\includegraphics[width=\textwidth]{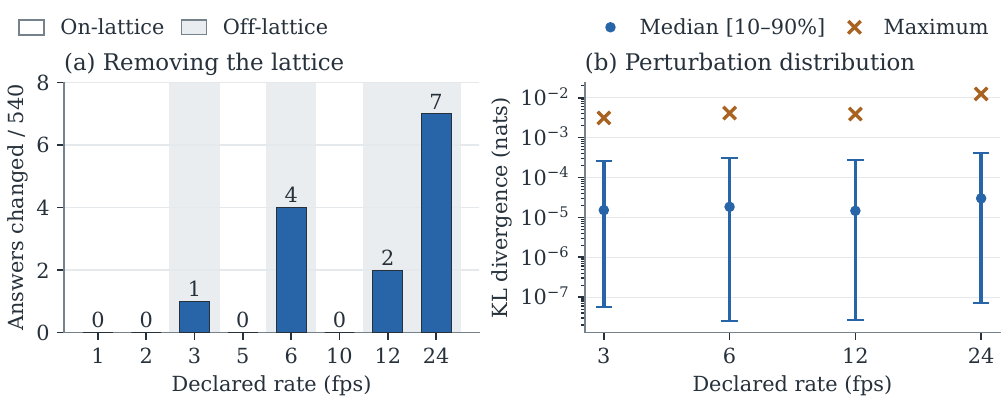}
\caption{\textbf{Only off-lattice declared rates perturb answers: 14 of 2{,}160, none on-lattice.} Left: changed answers at each
declared rate, with fixed image inputs and 540 paired questions per rate; off-lattice rates are
shaded. Right: median,
10th--90th percentiles, and maximum KL at the four rates that do not align with the lattice.
The intervals describe the distribution of perturbations, not confidence in a mean.}
\label{fig:lattice}
\end{figure}

Across the 2,160 evaluations at inexact rates, the median KL is approximately
$1.83\times10^{-5}$ nats and the maximum is 0.0124 nats. These are repeated evaluations of
540 questions, not 2,160 independent clips. The 14 answer changes across the complete grid
demonstrate a measurable positional perturbation; accuracy differences are unresolved.

\subsection{Overlapping spatial resolutions}\label{app:multiscale}
An exploratory probe combines four square-resized views of one decoded frame per clip, at side lengths
$(280,336,392,448)$ or $(196,280,392,560)$ pixels. Each configuration evaluates 400 questions.
If view $j$ produces $N_j$ image tokens, the measure arm assigns each a bias $-\log N_j$,
centered over image tokens while text bias remains zero. A reverse-group control reverses the four
resolution-level biases before recentering. It does not preserve their token-level histogram when
group sizes differ. This calibration also sets the relative image/text prior. Table~\ref{tab:multiscale} reports all completed arms.
The views overlap in physical space and are therefore not refinements of one partition;
their results bear on this implementation choice rather than on the exact partition theorem.
Neither size set resolves an accuracy difference from measure weighting.

\begin{table}[!htbp]
\centering\footnotesize
\caption{\textbf{Accuracy across overlapping resolutions.} Four hundred questions per configuration. Accuracy changes compare with standard attention within the same size set; intervals are paired and tests nominal. Each size set combines overlapping views and falls outside partition refinement.}
\label{tab:multiscale}
\begin{tabular}{@{}llrrr@{}}
\toprule
Size set & Arm & Accuracy (\%) & $\Delta$ (pp) [95\% CI] & $p$ \\
\midrule
Narrow & Standard & 54.25 & $+0.00$ & 1.000 \\
Narrow & Measure & 55.00 & $+0.75\;[-0.72,+2.22]$ & 0.508 \\
Narrow & Reverse groups & 55.75 & $+1.50\;[-0.20,+3.20]$ & 0.146 \\
Wide & Standard & 56.25 & $+0.00$ & 1.000 \\
Wide & Measure & 54.75 & $-1.50\;[-3.91,+0.91]$ & 0.307 \\
Wide & Reverse groups & 56.25 & $+0.00\;[-2.09,+2.09]$ & 1.000 \\
\bottomrule
\end{tabular}
\end{table}

\subsection{Pretrained support panel}\label{app:pretrained_support}
The support arm of Appendix~\ref{app:prospective} was also evaluated on the released checkpoint
without task training, on the same 2,489-question test grid and the same $K\in\{16,8,4,2\}$
groupings. These comparisons are descriptive; the confirmatory gates of
Appendix~\ref{app:prospective} apply to the trained panel only.
At the twofold budget, measure weighting with point support exceeds standard attention by
$+0.57$\,pp of mean accuracy and $+0.72$\,pp of all-partition-correct accuracy,
with mean pairwise JSD lower by $0.0022$ nats; the corresponding differences against calibrated
count are $+0.14$ and $+0.44$\,pp. Point support matches or exceeds the empirical
characteristic-function geometry for every arm and budget: under measure weighting, mean accuracy
is $62.99$\% versus $62.60$\% at twofold, $60.53$\% versus $58.23$\% at fourfold, and
$57.15$\% versus $54.68$\% at eightfold reduction. Accuracy degrades smoothly with the reduction
factor, from $65.05$\% at the native encoding to $57.15$\% at eightfold reduction under measure
weighting with point support. Decoder latency was recorded under two runtime builds and is
tabulated per build in the machine-readable report; the accuracy metrics pooled above do not
depend on the build.

\section{Deployment evaluation}\label{app:prospective}
This evaluation tests whether the weighting characterized by refinement symmetry improves
correctness across post-encoder partitions.
Table~\ref{tab:prospective} reports the primary deployment results; the support panel without
task training is reported in Appendix~\ref{app:pretrained_support}.

The core comparison holds a dense video encoding fixed and partitions its temporal tokens into
contiguous groups. Uniform and content-adaptive groups share total token budget and original-cell
coverage. Each group uses mass-weighted feature aggregation. Standard, global count, and
measure-weighted attention evaluate the same grouped features and physical positions. The
intervention changes post-encoder representations while holding raw-video encoding fixed.
A support arm uses the empirical characteristic function of inherited M-RoPE positions while
retaining the same native calibration.

\subsection{Acquisition, partitions, and training}\label{app:deployment_setup}
The 18-task MVBench panel contains 3,581 eligible questions from 3,290 source videos, with source-disjoint
splits of 696 training, 396 development, and 2,489 test questions (2,290 test videos).
The task families are action sequence, action prediction, action antonym, fine-grained action,
unexpected action, object existence, object interaction, object shuffle, moving direction,
action localization, scene transition, action count, moving count, moving attribute, state change,
character order, egocentric navigation, and counterfactual inference.
MVBench uses repeated question templates: 1,772 test questions share normalized text with another
test question, but the training, development, and test splits share no source videos. The deployment
adapters use only the 696 training questions; the exploratory LoRA study of
Appendix~\ref{app:trained} uses separate checkpoints and splits.
Fourteen missing media files and five single-alternative questions are excluded before outcomes.
Each input uses 32 uniform frames from the annotated interval, capped at 16 seconds and resized
to $224\times224$ pixels. The native encoding has 16 temporal patches and 64 spatial slots per
patch, or 1,024 visual tokens. We evaluate $K\in\{16,8,4,2\}$ contiguous temporal groups with
uniform, motion, entropy, and keyframe boundaries. Every grouping covers all native temporal patches.
Video biases are zero for standard attention, $\log(16/K)$ for global count, and $\log|G_i|$
for a measure-weighted group $G_i$; text biases remain zero.
Because native patches have equal mass, measure weighting here equals proportional attention
based on each group's inherited patch count~\citep{bolya2023tome} for every partition.
The global count comparator uses only the average group size, $16/K$. These rules coincide
for equal-size groups; their contrast on adaptive partitions tests the value of retaining
individual group masses, with features and point geometry held fixed.

Five shared checkpoints, with seeds zero through four, train only at the native partition.
The visual encoder is frozen. Query/value LoRA uses rank eight, scale two, and no dropout in the
text decoder; AdamW uses learning rate $2\times10^{-4}$, weight decay 0.01, accumulation over four
examples, gradient clipping at one, and 400 updates. All deployment rules use the same checkpoint
within a seed. Exact native-path controls precede endpoint analysis.

Training at the native partition and evaluating unseen groupings separates one-time model fitting
from deployment tokenization. Endpoints are mean and worst-partition accuracy, pairwise
Jensen--Shannon divergence across partitions, KL from the native answer distribution, agreement
across tokenizations, physical-token count, decoder latency,
and peak device memory. Source videos define clusters; repeated questions, budgets, and partitions
remain paired within their source. Multiple training seeds quantify training variability.

\subsection{Primary endpoints and inference}\label{app:deployment_inference}
The primary budget is $K=8$. Four two-sided contrasts compare measure with count and with standard
attention, for (i) correctness under every partition of a question and (ii) mean pairwise JSD
between its partition-specific answer distributions. The primary budget and four contrasts were
specified in the suite's report generator before data collection. The supplied reports record
the report-generator SHA-256 (\texttt{023b156d}$\ldots$\texttt{089849}) in their run headers.
This source record documents prespecification; the study has no external preregistration.
The development split supplied a two-question acquisition/throughput
pilot and was not used for hyperparameter selection or model comparison.
These endpoints differ from the minimum accuracy
across the four partitions, which Table~\ref{tab:perpartition} reports descriptively. A crossed source-video/seed bootstrap
with 20,000 draws gives descriptive intervals. Each contrast also uses a paired source-cluster
sign-flip test after averaging seeds and a paired five-seed $t$-test after averaging test questions.
The sign-flip test draws 20{,}000 Monte Carlo sign assignments, so its attainable minimum
$p$-value is $1/20{,}001\approx5\times10^{-5}$; each contrast gates on the larger of its two
$p$-values. This floor determines none of the four final gated $p$-values on MVBench.
The adjusted value is $\min\{1,4\max(p_{\rm source},p_{\rm seed})\}$, applying
four-comparison Bonferroni control within each benchmark.
The source test assesses source-level effects conditional on the evaluated checkpoints;
the seed test assesses training variability conditional on the test set. The primary gate requires
both tests to pass the four-comparison threshold. Source sign flips assume independent videos
and joint sign symmetry of paired cluster effects under the null; the five-seed $t$-test assumes
independent, approximately normal seed-level mean differences. The gate combines these two
conditional tests; inference over jointly unseen videos and training runs requires a jointly
calibrated test. The crossed bootstrap intervals describe joint resampling variability.
Training variability is quantified by the five seeds. Other budgets, support comparisons, and
pretrained-only panels are exploratory.

\begin{table}[t]
\centering\small
\caption{\textbf{Merging cuts decoder latency far more than total cost.} Measured inference costs for the point-support panel. Stage sum includes acquisition, saliency, grouping and decoder time; acquisition dominance keeps the stage-sum saving at 5.2--5.6\% from the native partition to eightfold merging. Peak is allocated CUDA memory, including the frozen dense encoder.}
\label{tab:deployment-cost}
\begin{tabular}{llrrr}
\toprule
Reduction & Rule & Decoder (ms) & Stage sum (ms) & Peak (GiB) \\
\midrule
1$\times$ & Standard & 303.6 & 4710.5 & 19.34 \\
1$\times$ & Global count & 294.9 & 4698.4 & 19.34 \\
1$\times$ & Measure & 285.4 & 4686.3 & 19.34 \\
2$\times$ & Standard & 143.1 & 4523.8 & 19.34 \\
2$\times$ & Global count & 147.5 & 4530.0 & 19.34 \\
2$\times$ & Measure & 152.0 & 4536.1 & 19.34 \\
4$\times$ & Standard & 105.6 & 4484.8 & 19.34 \\
4$\times$ & Global count & 106.0 & 4485.1 & 19.34 \\
4$\times$ & Measure & 106.2 & 4485.5 & 19.34 \\
8$\times$ & Standard & 73.1 & 4444.9 & 19.34 \\
8$\times$ & Global count & 73.1 & 4445.1 & 19.34 \\
8$\times$ & Measure & 72.8 & 4444.9 & 19.34 \\
\bottomrule
\end{tabular}
\end{table}

At the native encoding, mean test accuracy is 66.83\%. Under twofold reduction, measure weighting
attains 64.82\% mean accuracy and 57.44\% all-partition-correct accuracy
(Table~\ref{tab:prospective}). Three of the four prespecified contrasts are significant after
Bonferroni adjustment (Table~\ref{tab:deployment-inference}): all-partition-correct accuracy rises
by $+1.04$\,pp against global count ($p = 0.008$) and by $+0.93$\,pp against standard
attention ($p = 0.024$), and mean pairwise JSD falls by $0.0063$ nats against global count
($p = 7.7\times10^{-4}$). The JSD difference from standard attention remains unresolved
at $-0.0010$ nats ($p = 0.90$).
Decoder latency falls from 0.29--0.30\,s at the native encoding, depending on the rule, to
0.073\,s at eightfold reduction,
while peak device memory remains 19.34\,GiB at every budget (Table~\ref{tab:deployment-cost}).
The latency reduction reflects shorter decoder sequences; encoder cost is unchanged.

\subsection{A second benchmark: WorldSense}\label{app:worldsense}
The identical panel extends to a second benchmark. The WorldSense manifest excludes seventeen
questions with ambiguous or invalid answer alternatives and holds 3,155 questions from 1,657
source videos, with source-disjoint splits of 667 training, 300 development, and 2,188 test
questions over 1,158 test videos. Five separately trained checkpoints, seeds zero through four, train only at
the native partition on the 667 training questions under the protocol's training configuration,
and the same suite evaluates the identical grid on the test split. Acquisition, budgets,
partitions, attention rules, training configuration, and the four-contrast analysis match the
MVBench protocol. The WorldSense manifest has SHA-256
\texttt{bcc2df28}$\ldots$\texttt{acdc2b}; its run headers contain the same suite and base-checkpoint
hashes. Measurements use NVIDIA A800-SXM4-80GB GPUs. This panel reports accuracy and distributional
endpoints; its costs are not pooled with the H20 measurements. The development split is unused.
The calibration and cross-rate experiments in Appendices~\ref{app:released} and~\ref{app:deployment}
are separate exploratory studies.

All three rules coincide at the native encoding, at 39.53\% mean test accuracy; the benchmark is
substantially harder for this model than MVBench (66.83\%). At the primary budget, measure
weighting attains 31.73\% robust accuracy against 30.99\% for global count and 31.52\% for
standard attention, and mean pairwise JSD of 0.0331 nats against 0.0380 and 0.0333
(Table~\ref{tab:prospective}).
Robust accuracy rises by $+0.74$\,pp against global count ($p = 0.016$) and mean pairwise
JSD falls by $0.0049$ nats against global count ($p = 2\times10^{-4}$; here the sign-flip
Monte Carlo floor determines the final gated $p$-value).
Against standard attention, the robust-accuracy gain is $+0.21$\,pp with an interval spanning
zero ($p = 1.0$), compared with $+0.93$\,pp on MVBench. The JSD contrasts against standard
attention remain unresolved (WorldSense: $-0.0002$; MVBench: $-0.0010$ nats).
Both benchmarks therefore support groupwise over average mass weighting for all-partition
correctness and JSD at the primary budget.

\subsection{Representation and allocation diagnostics}\label{app:mechanism}
A diagnostic panel examines how merging changes input representations and attention allocation.
It uses the seed-0 trained adapter and 400 test questions from 400 distinct source videos,
stratified over 18 tasks. For the native encoding and all 36 merged configurations, the panel
records native video embeddings and the readout query's key scores at every decoder layer and
head: $400\times36\times28\times28\approx11.3$ million
question--configuration--layer--head comparisons.

The first diagnostic is the maximal input-embedding deviation
$\epsilon_h=\max_{j,\,i\in G_j}\|h_i-\bar h_{G_j}\|_2$, where $\bar h_{G_j}$ is a group's mean
native embedding at a matched spatial slot. It is identical across attention rules because their input groups are shared.
Its mean increases with reduction on every partition (Table~\ref{tab:mechanism}); for uniform
groups it rises from 45.5 to 65.3 and 74.7 at two-, four-, and eightfold reduction.
The last value is 55\% of the mean native video-embedding diameter $D_h=135.9$.
Across partitions, 94.5--100\% of questions show a strict increase from twofold to eightfold.
These are input-representation measurements, distinct from the layerwise value errors in
Theorem~\ref{thm:stability}.

The second diagnostic compares attention distributions on a common index set. Spread each merged
key's attention probability uniformly over its inherited native patches, and align text keys by
sequence order. Let $P$ be the native readout-row distribution, $\widetilde Q$ the resulting
merged distribution, $T=\operatorname{TV}(P,\widetilde Q)$, and
$\Delta_{\rm attn}=\operatorname{osc}\log(\widetilde Q/P)$.
Lemma~\ref{lem:tilt} gives $T\le\tanh(\Delta_{\rm attn}/4)$ for each layer and head.
Every recorded comparison satisfies this inequality; the largest residual
$T-\tanh(\Delta_{\rm attn}/4)$ is $-4.2\times10^{-5}$.
This checks the change-of-distribution inequality, independently of the values.
For measure weighting, the pullback also agrees with the physical-mass coupling of
Theorem~\ref{thm:stability}; for the other rules, $\Delta_{\rm attn}$ includes their departure
from inherited mass weighting.

Table~\ref{tab:mechanism} scales $T$ by the per-question input-embedding diameter $D_h$ for a
common descriptive scale. This diagnostic is largest under standard attention in every
budget--partition cell. Under measure weighting it is no larger than under global count weighting in
eleven of twelve cells, including the three exact equalities on uniform partitions; the exception
is a 0.04\% inversion for the eightfold entropy partition.
Among the 271 natively correct questions, the fraction answered incorrectly under at least one
merged partition rises from 15.5--17.0\% at twofold to 29.9--31.7\% at eightfold reduction.
Those questions have higher $D_hT$ than questions remaining correct in all nine budget--rule
comparisons (12.6 versus 10.2 for measure weighting at eightfold). Separation by $\epsilon_h$
and the recorded score range varies across budgets.

These associations describe input embeddings and attention allocation for one adapter seed.
Here $D_hT$ supplies a common descriptive scale for attention changes. A bound on a decoder
layer's output requires values and diameters from that same layer; a network bound also requires
the complete block discrepancies in Corollary~\ref{cor:network}. Thus
$\epsilon_h+D_hT$ remains a diagnostic quantity rather than a validated decoder bound.
The panel's cell standard error is approximately 2.4\,pp. The paired five-seed deployment study
quantifies the roughly one-point accuracy differences; explaining their budget dependence
requires further analysis of the decoder's layerwise errors.

\begin{table}[!htbp]
\centering\small
\caption{\textbf{Input perturbation and attention redistribution under merging.}
Mean maximal input-embedding deviation $\epsilon_h$ (400 questions; standard error $\le0.72$;
identical across rules) and the embedding-scaled attention diagnostic $D_hT$, averaged over
layers and heads (ranges across four partitions). The mean native video-embedding diameter
is $D_h=135.9$. The quantities in the first two blocks have input-embedding units.}
\label{tab:mechanism}
\begin{tabular}{lrrr}
\toprule
 & $2\times$ & $4\times$ & $8\times$ \\
\midrule
\multicolumn{4}{l}{\textit{Input-embedding deviation $\epsilon_h$ by partition}} \\
uniform  & 45.5 & 65.3 & 74.7 \\
motion   & 55.8 & 66.6 & 74.8 \\
entropy  & 51.5 & 66.1 & 75.0 \\
keyframe & 62.2 & 71.0 & 75.6 \\
\midrule
\multicolumn{4}{l}{\textit{Embedding-scaled attention TV, $D_hT$, by rule (partition range)}} \\
Standard & 7.68--8.17 & 11.46--11.83 & 14.47--14.54 \\
Global count    & 6.11--7.07 & 8.56--9.42   & 10.70--11.04 \\
Measure  & 6.11--6.51 & 8.56--8.85   & 10.68--10.78 \\
\bottomrule
\end{tabular}
\end{table}

\subsection{Per-partition decomposition}\label{app:perpartition}
Table~\ref{tab:perpartition} decomposes the trained panel by
allocation partition (597{,}360 measurement rows; 2{,}489 held-out questions, five training
seeds).

\begin{table}[!ht]
\centering\small
\caption{\textbf{Measure-versus-global-count differences arise on adaptive partitions.} Pooled MVBench test accuracy (\%) for
Qwen2.5-Omni-7B at $K=8$, over 12{,}445 question--seed pairs per cell.
Differences are measure minus comparator in pp, computed before rounding. On the uniform
partition the measure and count rules coincide exactly. The final row tabulates the
worst-partition-accuracy endpoint, which the keyframe partition attains for every arm at this
budget.}
\label{tab:perpartition}
\begin{tabular}{@{}lrrrrr@{}}
\toprule
Partition & Standard & Global count & Measure & \shortstack{$\Delta$ vs. std.\\(pp)} & \shortstack{$\Delta$ vs. count\\(pp)} \\
\midrule
uniform  & 64.94 & 65.42 & 65.42 & +0.47 & 0.00 \\
motion   & 64.19 & 64.36 & 64.69 & +0.51 & +0.33 \\
entropy  & 64.37 & 65.01 & 65.12 & +0.75 & +0.11 \\
keyframe & 63.03 & 63.78 & 64.07 & +1.04 & +0.28 \\
\midrule
worst partition (keyframe) & 63.03 & 63.78 & 64.07 & +1.04 & +0.28 \\
\bottomrule
\end{tabular}
\end{table}

Under the uniform partition, the measure and count rules produce equal-size groups and
bitwise-identical answer logits in all 12{,}445 question--seed pairs at every budget, and their
accuracies coincide exactly (difference 0.0), confirming the analytical identity between the two
rules under uniform allocation. Measure-versus-global-count differences therefore arise entirely on adaptive partitions.
At the primary budget of eight groups, the per-partition accuracy gain of measure over the
standard rule is largest on the keyframe partition: $+0.51$ (motion), $+0.75$ (entropy), and
$+1.04$\,pp (keyframe). Against the count rule the ordering differs: the gains are $+0.33$
(motion), $+0.11$ (entropy), and $+0.28$\,pp (keyframe), largest on motion, and at deeper budgets
the keyframe gain over count turns negative ($-0.42$\,pp at fourfold and $-0.20$\,pp at eightfold
reduction). The keyframe partition is the hardest for every arm at the primary budget, and the
measure rule lifts that worst partition over both comparators, from 63.03 (standard) and 63.78
(count) to 64.07 (measure). This is a descriptive worst-partition result, distinct from the
all-partition-correct endpoint.

\end{document}